\documentclass[sigconf]{acmart}
\pdfoutput=1
\AtBeginDocument{%
  }

\copyrightyear{2024}
\acmYear{2024}
\setcopyright{acmlicensed}\acmConference[KDD '24]{Proceedings of the 30th ACM SIGKDD Conference on Knowledge Discovery and Data Mining}{August 25--29, 2024}{Barcelona, Spain}
\acmBooktitle{Proceedings of the 30th ACM SIGKDD Conference on Knowledge Discovery and Data Mining (KDD '24), August 25--29, 2024, Barcelona, Spain}
\acmDOI{10.1145/3637528.3671731}
\acmISBN{979-8-4007-0490-1/24/08}

\usepackage{microtype}
\usepackage{graphicx}
\usepackage{subfigure}
\usepackage{booktabs} 
\usepackage{colortbl}
\usepackage{ulem} 

\usepackage{hyperref}
\usepackage{url}

\usepackage{amsmath}
\usepackage{mathtools}
\usepackage{amsthm}

\usepackage[capitalize,noabbrev]{cleveref}

\theoremstyle{plain}
\newtheorem{theorem}{Theorem}[section]
\newtheorem{proposition}[theorem]{Proposition}
\newtheorem{lemma}[theorem]{Lemma}

\theoremstyle{definition}
\newtheorem{definition}[theorem]{Definition}

\theoremstyle{remark}
\newtheorem{remark}[theorem]{Remark}

\begin{document}

\title{An Efficient Subgraph GNN with Provable Substructure Counting Power}

 \author{Zuoyu Yan}
 \orcid{0000-0003-2585-3838}
 \email{yanzuoyu3@pku.edu.cn}
 \affiliation{%
  \institution{Wangxuan Institute of Computer Technology, Peking University }
  \city{Beijing}
  \country{China}
}

 \author{Junru Zhou} 
 \orcid{0000-0002-1192-7712}
 \email{zml72062@stu.pku.edu.cn}
 \affiliation{%
  \institution{Institute for Artificial Intelligence, Peking University }
  \city{Beijing}
  \country{China}
}
 \author{Liangcai Gao}
 \authornote{Correspondence to Muhan Zhang, and Liangcai Gao}
 \orcid{0000-0002-5196-4825}
 \email{glc@pku.edu.cn}
 \affiliation{%
  \institution{Wangxuan Institute of Computer Technology, Peking University }
  \city{Beijing}
  \country{China}
}
 \author{Zhi Tang}
 \email{tangzhi@pku.edu.cn}
 \orcid{0000-0002-6021-8357}
 \affiliation{%
  \institution{Wangxuan Institute of Computer Technology, Peking University }
  \city{Beijing}
  \country{China}
}
\author{Muhan Zhang}
 \orcid{0000-0002-7680-6401}
 \authornotemark[1]
 \email{muhan@pku.edu.cn}
 \affiliation{%
  \institution{Beijing Institute for General Artificial Intelligence, Peking University  }
  \city{Beijing}
  \country{China}
}

\renewcommand{\shortauthors}{Yan et al.}

\begin{abstract} 
We investigate the enhancement of graph neural networks' (GNNs) representation power through their ability in substructure counting. Recent advances have seen the adoption of subgraph GNNs, which partition an input graph into numerous subgraphs, subsequently applying GNNs to each to augment the graph's overall representation. Despite their ability to identify various substructures, subgraph GNNs are hindered by significant computational and memory costs. In this paper, we tackle a critical question: Is it possible for GNNs to count substructures both \textbf{efficiently} and \textbf{provably}? Our approach begins with a theoretical demonstration that the distance to rooted nodes in subgraphs is key to boosting the counting power of subgraph GNNs. To avoid the need for repetitively applying GNN across all subgraphs, we introduce precomputed structural embeddings that encapsulate this crucial distance information. Experiments validate that our proposed model retains the counting power of subgraph GNNs while achieving significantly faster performance.
\end{abstract}

\ccsdesc[500]{Computing methodologies~Neural networks}
\ccsdesc[500]{Computing methodologies~Supervised learning by classification}
\ccsdesc[500]{Computing methodologies~Supervised learning by regression}

\keywords{Subgraph GNN, Count substructures, Subgraph Isomorphism counting, Graph Isomorphism test, Graph Neural Network}

\maketitle

\section{Introduction}

Message Passing Neural Networks (MPNNs) are the most commonly used GNNs. They have achieved remarkable success on graph representation learning~\citep{KipfW17, xu2018powerful}, and have been widely used in various downstream tasks~\citep{wu2020comprehensive, zhou2020graph}. However, the representation power of MPNNs is limited~\citep{xu2018powerful}, thus increasing efforts have been spent on designing more powerful GNNs.

An essential and intuitive way to evaluate the representation power of GNNs is whether they \textbf{provably} approximate specific functions, such as counting graph substructures. Substructures often represent meaningful components in graphs and can reveal essential structural insights in chemistry~\citep{deshpande2002automated, jin2018junction}, biology~\citep{koyuturk2004efficient}, and sociology~\citep{jiang2010finding}. For instance, in a molecule graph, the presence of a 6-cycle (hexagon) passing through a node suggests its potential association with a benzene ring. 
Moreover, substructures are closely connected to many fundamental graph properties~\citep{shervashidze2009efficient, preciado2010local}. 
In graph representation learning tasks, many targets are unknown or intractable functions of the graph structure which we need to learn from the data. However, MPNNs' substructure counting ability is shown to be very limited, failing to count even triangles~\citep{chen2020can}. If the underlying graph functions depend on some substructures that the GNN theoretically cannot detect/count, we can never trust its performance on these tasks.

In this paper, we focus on the counting power of GNNs, i.e., whether GNNs can provably count the number of given connected substructures within a graph. Notice that we are not designing models to only count substructures (which can be perfectly done by traditional algorithms), but to 
propose GNNs with high enough representation power to solve unknown structure-related tasks that we need to learn from the data. In classic works~\citep{furer2017combinatorial}, to reach a high representation power, globally expressive models such as 3-WL~\citep{maron2019provably, tahmasebi2020counting} are needed. However, their high computational costs restrict their universal use in real-world applications.

Recently, a series of works called subgraph GNNs~\citep{you2021identity, zhang2021nested,bevilacqua2022equivariant,zhao2022stars,frascaunderstanding2022,huang2023boosting} are shown to provably count certain substructures. For an input graph, they first decompose the input graph into a collection of subgraphs (overlap is allowed) based on certain subgraph selection policies. Then some base GNNs are applied to the extracted subgraphs whose representations are used to enhance the graph representation. Although faster than globally expressive GNNs, they still need to run GNNs over all subgraphs. Therefore they are much slower compared with classic MPNNs, and suffer from high computational cost when encoding large and dense graphs. 

Based on the above observation, we raise a nontrivial question: can we count substructures \textbf{efficiently and provably} with GNNs, ideally using a similar cost to MPNNs? To answer the above, we decompose it into two sub-questions: (1) what provides the extra representation power of subgraph GNNs compared with classic MPNNs? (2) can we utilize such information efficiently without running GNNs on subgraphs? To answer the first question, we show that the key to boosting the counting power of subgraph GNNs is the \textbf{distance} to the rooted nodes within the subgraph. To answer the second question, we find that such distance information can be encoded into a precomputed structural embedding. Combining it with a base GNN, there is no need to run GNNs repeatedly on each subgraph.  
In this way, we only need to run GNNs on the original graph (augmented with precomputed structural embeddings), while being able to efficiently count substructures. 

In summary, our contributions are listed as follows:

(1) We theoretically characterize the general substructure counting power of subgraph GNNs, showing that they are \textit{much more efficient yet nearly as powerful} as globally expressive models.

(2) To accelerate subgraph GNNs, we theoretically show that the distance to the rooted nodes within subgraphs is key to boosting their representation power. We then propose a structural embedding to encode such distance information. 

(3) We propose a model, Efficient Substructure Counting GNN (\textit{ESC-GNN}), which enhances a basic GNN with the structural embedding. It only needs to run message passing on the whole graph, and thus is much more efficient than subgraph GNNs. We evaluate ESC-GNN on various real-world and synthetic benchmarks. Experiments show that ESC-GNN performs comparably with subgraph GNNs on real-world tasks and counting substructures, while running much faster.

\section{Related Works}
\label{sec:rel}
\subsection{Representation power of GNNs} 

There are two major perspectives to evaluate the representation power of GNNs: the ability to distinguish non-isomorphic graphs, and the ability to approximate specific functions. In terms of differentiating graphs, existing works~\citep{xu2018powerful, morris2019weisfeiler} showed that MPNNs are at most as powerful as 1-WL~\citep{weisfeiler1968reduction}. Following works improve the expressiveness of GNNs by using high-order information~\citep{morris2019weisfeiler, morris2020weisfeiler, maron2019provably, bodnar2021weisfeiler, bodnar2021weisfeiler1, vignac2020building, yan2022cycle, wang2023message} or augmenting node features~\citep{bouritsas2022improving, barcelo2021graph, dwivedi2021graph, loukas2020hard, abboud2021surprising, kreuzer2021rethinking, lim2022sign, michel2023path, lim2023expressive, chen2023improving}. 

In terms of approximating specific functions, some works use GNNs to approximate graph algorithms~\citep{velivckovic2019neural, xhonneux2021transfer, yan2022neural} or detect bi-connectivity~\citep{zhang2023rethinking}. In this paper, we focus on counting substructures. Previous works~\citep{furer2017combinatorial, arvind2020weisfeiler} relate the counting power to the expressiveness of GNNs. Following works count substructures with globally expressive  networks~\citep{murphy2019relational, chen2020can, tahmasebi2020counting} that are with high computational costs. Several other works~\citep{liu2020neural, liu2022graph, yu2023learning} also focus on counting substructures. However, they do not provide theoretical guarantees for these counts. Therefore we cannot trust their performance on substructure-related tasks, which are abundant in chemistry and biology.

\subsection{Subgraph GNNs} 
Subgraph GNNs can be divided according to their subgraph selection policies, such as graph element deletion~\citep{bevilacqua2022equivariant, cotta2021reconstruction, papp2021dropgnn}, k-hop subgraph extraction~\citep{abu2019mixhop, sandfelder2021ego, nikolentzos2020k, feng2022powerful, yao2023improving}, node identity augmentation~\citep{you2021identity}, and rooted subgraph extraction~\citep{zhang2018link, zhang2021nested, zhao2022stars, frascaunderstanding2022, papp2022theoretical, zhang2021labeling, huang2023boosting, qianordered2022, yang2023extract, bevilacqua2024efficient}. Most subgraph GNNs need to run message passing over all subgraphs, therefore performing much slower than classic MPNNs. This prevents their use in large real-world datasets. 

\subsection{Positional/Structural Encodings on Graphs} 

To leverage the spectral properties of graphs, many works~\citep{dwivedi2020benchmarking, lim2022sign, dwivedi2021graph, kreuzer2021rethinking, mialon2021graphit, park2022grpe, rampavsek2022recipe} introduce the eigenvectors of the graph Laplacian as augmented node features. Other approaches introduce encodings with essential graph features~\citep{li2020distance, feldman2022weisfeiler, ying2021transformers, wang2022equivariant, yan2021link, bouritsas2022improving, keriven2023what, zhou2023improving, zhu2023structural, yan2023cycle, huang2023stability}.  
Our work can be viewed as a subgraph-based structural encoding.

\section{Preliminaries}
Let $G = (V,E)$ be a simple, undirected graph, where $V = \{1,2,...,N\}$ is the node set, and $E$ is the edge set. We use $x_v$ to represent the node attribute for $v \in V$, and $e_{uv}$ to represent the edge attribute for $uv \in E$. Denote the $h$-hop neighborhood of node $v$ as $V_{v}^h = \{u \in V|d(u,v) \leq h\}$, where $d(u,v)$ denotes the shortest path distance between node $u$ and $v$. For a special case where $h = 1$, we call it the neighborhood of node $v$: $N(v) = V_v^1$.

Define a subgraph of $G$ as any graph $G^S = (V^S, E^S)$ with $V^S \subseteq V$ and $E^s \subseteq E$. And an induced subgraph of $G$ is any graph $G^I = (V^I, E^I)$ where $V^I \subset V$, and $E^I = E \cap (V^I)^2$ is the set of all edges in $E$ where both nodes belong to $V^I$. For a $k$-tuple $\vec{v} = (v_1, ..., v_k) \in (V)^k$, define its rooted $h$-hop subgraph as $G_{\vec{v}}^h = (V_{\vec{v}}^h, E_{\vec{v}}^h)$, where $V_{\vec{v}}^h$ is the union of the $h$-hop neighborhoods of all vertices in $\vec{v}$: $V_{\vec{v}}^h = V_{v_1}^h \cup ... \cup V_{v_k}^h$, and $E_{\vec{v}}^h = E\cap (V_{\vec{v}}^h)^2$ are edges whose two nodes both belong to $V_{\vec{v}}^h$. Later, We will omit the hop parameter $h$ for simplicity.

In this paper, we focus on graph-level counting of connected substructures\footnote{Some works~\citep{huang2023boosting} focus on node-level counting, which can also be transferred to graph-level counting.}. Connected substructures are substructures whose nodes belong to a connected component. We study four types of connected substructures that are widely used in existing works: cycles, cliques, stars, and paths. An $L$-path is a sequence of edges $[(v_1, v_2),...,(v_L,v_{L+1})]$ such that all nodes are distinct; an $L$-cycle is an $L$-path except that $v_1 = v_{L+1}$; an $L$-clique is a fully connected graph with $L$ nodes; and an $L$-star denotes a set of edges $[(v, v_1), (v, v_2),...,(v, v_{L-1})]$ where all nodes with different symbols are distinct. Two substructures are called equivalent if their sets of edges are equal. Given a substructure $S$ and a graph $G$, the \textbf{subgraph counting} is defined as counting the number of inequivalent substructures $C_S(S,G)$ that are subgraphs of $G$. The \textbf{induced subgraph counting} is defined as counting the number of inequivalent substructures that are induced subgraphs of $G$. In this paper, we focus on subgraph counting, but we also provide theoretical results on induced subgraph counting.

Following existing works~\citep{chen2020can, huang2023boosting}, we formally define substructure counting as:

\begin{definition}
Let $\mathcal{G}$ be the set of all graphs and $\mathcal{F}$ be a function class over graphs. We say $\mathcal{F}$ can count connected substructure $S$ on $\mathcal{G}$ if for all $G_1$, $G_2 \in \mathcal{G}$ such that $C_S(S, G_1) \neq C_S(S, G_2)$, their exists $f \in \mathcal{F}$ that $f(G_1) \neq f(G_2)$.
\end{definition}

Using the Stone-Weierstrass theorem, the definition is equivalent to approximating subgraph-counting functions~\citep{chen2020can}. If replacing $C_S$ by $C_I$, then the task will be naturally 
turned to induced subgraph counting.

\section{Counting Power of Subgraph GNNs}
\label{sec:power}

Subgraph GNNs have been proven to possess the capability to count specific substructures~\citep{chen2020can,you2021identity,zhao2022stars, huang2023boosting}. 
We commence by introducing subgraph GNNs in Section~\ref{subsec:subgnn}. In Section~\ref{subsec:count}, we characterize the general substructure counting power of subgraph GNNs, by proving that they are nearly as powerful as globally expressive models while running much faster. We theoretically show that the distance information within the subgraph is key to boosting the counting power of GNNs. Such information can be encoded into a structural embedding, providing the basis for our proposed efficient substructure-counting model\footnote{The distance-related information also strongly influences the expressiveness of subgraph GNNs since the subgraph selection policies and the unsymmetric treatment to the rooted node and other nodes are both based on such information.}.

\subsection{Subgraph GNNs}
\label{subsec:subgnn}

In this section, we introduce subgraph GNNs and refer readers to the appendix for details on the Weisfeiler-Leman algorithm (WL)~\citep{weisfeiler1968reduction} and MPNNs.

Subgraph GNNs first represent the input graph by a collection of subgraphs based on certain subgraph selection policies. They then encode the subgraphs using backbone GNNs and aggregate subgraph representations into the graph representation. We note that there exist some other variants of subgraph GNNs~\citep{qianordered2022, zhao2022stars, frascaunderstanding2022, bevilacqua2022equivariant}, but in this paper, we focus on a specific type of subgraph GNNs without information exchange between subgraphs, which covers reconstruction GNNs~\citep{papp2021dropgnn, cotta2021reconstruction}, ID-GNNs~\citep{you2021identity}, and nested GNNs~\citep{ zhang2021nested,huang2023boosting}. We will show that this type of subgraph GNNs is powerful enough in terms of counting connected substructures. 

We call their subgraph selection policy as \textit{rooted subgraph extraction policy}. Existing works select subgraphs rooted at either nodes~\citep{zhang2021nested, you2021identity} or $k$-tuples~\citep{huang2023boosting, qianordered2022}, and typically use a $1$-WL equivalent GNN as the backbone. In this paper, we propose a more general framework with \textit{$m$-WL (or its equivalent GNN) as the backbone on subgraphs rooted at connected $k$-tuples}, i.e., $k$-tuples in which their induced subgraphs are connected. For a connected $k$-tuple $\vec{v}$, the selected subgraph is its rooted subgraph $G_{\vec{v}}=(V_{\vec{v}}, E_{\vec{v}})$.

Denote $c^t_{\vec{v},\vec{u}}$ as the color for an $m$-tuple $\vec{u} = (u_1,...,u_m)$ in $G_{\vec{v}}$ at iteration $t$. It is computed by:
\begin{equation}
\label{eq:submpnn}
c_{\vec{v}, \vec{u}}^t = \text{HASH}(c_{\vec{v}, \vec{v}}^{t-1}, c_{\vec{v}, \vec{u}}^{t-1}, c_{\vec{v}, \vec{u}, (1)}^{\textcolor{black}{t}},...,c_{\vec{v}, \vec{u}, (k)}^{\textcolor{black}{t}})
\end{equation}
where 
\begin{equation}
\label{eq:submpnn2}
c_{\vec{v}, \vec{u}, (i)}^t = \{\!\!\{c_{\vec{v},\vec{q}}^{t-1}|\vec{q} \in N_{\vec{v},i}(\vec{u})\}\!\!\}, i \in [m]
\end{equation}
Here $N_{\vec{v}, i}(\vec{u}) = \{(u_1,...,u_{i-1},w,u_{i+1},...,u_m)| w\in V_{\vec{v}}\}$ denotes the $i$-th neighborhood of $\vec{u}$ in $G_{\vec{v}}$.

Denote the final iteration as iteration $T$. The color of $\vec{v}$ after the final iteration will be the combination of all $m$-tuples' colors inside $G_{\vec{v}}$. Formally, 
\begin{equation}
c_{\vec{v}} = Readout(\{\!\!\{c_{\vec{v},\vec{u}}^T|\vec{u} \in (V_{\vec{v}})^m\}\!\!\})
\end{equation}
where $Readout$ is a readout function, e.g., the sum function. Intuitively, compared with MPNNs, subgraph GNNs (1) update the representation using not only the neighboring information but also the information from the rooted nodes; (2) the neighborhood is restricted to the subgraph level.

\subsection{Counting Power of Subgraph GNNs}
\label{subsec:count}
Subgraph GNNs possess the ability to count substructures. Existing works mainly focus on counting certain types of substructures, e.g., walks~\citep{you2021identity} and cycles~\citep{huang2023boosting} and do not relate subgraph GNNs with substructure counting in a holistic perspective. In this section, we characterize the general substructure counting power of subgraph GNNs by showing that they are nearly as powerful as globally expressive models, e.g., high-dimensional WL, while running much faster. We first theoretically characterize $m$-WL's power for counting \textbf{any} substructures.

\textbf{Counting power of $m$-WL.} Different substructures with no more than $m$ nodes have different initial isomorphic types. We can assign each isomorphic type a unique color, and define the color histogram of the graph as the output function. Therefore the lower bound of the counting power of $m$-WL is:
\begin{remark}
\label{rm:wl}
$m$-WL ($m \geq 2$) can count all connected substructures with no more than $m$ nodes.
\end{remark}

As for the upper bound, we show that for any $m \geq 2$, there exists a type of connected substructure with $m + 1$ nodes that $m$-WL cannot count. The theorem is formally stated below, and the proofs of the following theorems are provided in the appendix.

\begin{theorem}
    \label{th:disab}
	For any $m\geq 2$, there exists a pair of graphs $G$ and $H$, such that $G$ contains an $(m+1)$-clique as its subgraph while $H$ does not, and that $m$-WL cannot distinguish $G$ from $H$.
\end{theorem}

Theorem~\ref{th:disab} makes the lower bound in Remark~\ref{rm:wl} tight.

\textbf{Decomposition of counting connected substructures.} Before discussing the counting power of subgraph GNNs, we first provide the basis for the discussion: any graph-level substructure counting can be naturally decomposed into a collection of local substructure counting. For example, to count 3-cliques/3-cycles in a given graph, we can first compute the number of 3-cycles that pass each node, sum the number among all nodes, and then divide the number by 3~\footnote{This number is strongly related to graph automorphism and is specific to the type of structure.} to compute the graph-level result. We can extend the observation to a more general version:

\begin{remark}
\label{rm:de}
To count a certain type of connected substructures with no more than $m+k$ ($m \geq 2, k > 0$) nodes in a given graph, we can decompose it into counting over $k$-tuples. First, select a specific type of connected $k$-tuple whose induced subgraph is a subgraph of the target substructure. Then count the substructures that pass each $k$-tuple. Finally, the result is the sum of the numbers over all $k$-tuples divided by a constant dependent on the substructure.
\end{remark} 

\textbf{Counting power of subgraph GNNs.} We first give the lower bound below.

\begin{theorem}
\label{th:sub}
For any connected substructure with no more than $m+k$ ($m\geq 2, k > 0$) nodes, there exists a subgraph GNN rooted at $k$-tuples with backbone GNN as powerful as $m$-WL that can count it. 
\end{theorem}

Based on the proof and the insights gained from popular subgraph GNNs, we can safely conclude that \textit{the distance from the rooted nodes to nodes within the subgraph provides valuable information for substructure counting}. As for the upper bound of the counting power of subgraph GNNs, existing works~\citep{geerts2020expressive, frascaunderstanding2022} show that for $m \geq 2$, (1) $m$-WL is as powerful as a specific GNN, called $m$-IGN~\citep{maron2018invariant}; (2) a subgraph GNN rooted at $k$-tuple
with backbone GNN as powerful as $m$-WL can be implemented by a $(m+k)$-IGN (which can be easily induced by Lemma 5 in~\citep{frascaunderstanding2022}). Therefore, we have:

\begin{remark}
\label{pro:sub}
Any subgraph GNN rooted at $k$-tuples with backbone GNN as powerful as $m$-WL ($m \geq$ 2) is not more powerful than $(m+k)$-WL.
\end{remark}
Combining with Theorem~\ref{th:disab}, we obtain a tight characterization of subgraph GNNs' counting power for any substructures, which is the same as $(m+k)$-WL.
This suggests that subgraph GNNs are nearly as powerful as globally expressive models in terms of graph-level general substructure counting. Note that globally expressive models might still be more powerful at counting specific substructures.

\textbf{Efficiency of subgraph GNNs.}  The computational cost for 
$(m+k)$-WL is $O(|V|^{m+k})$, where $|V|$ denotes the number of nodes in the input graph. However, for subgraph GNNs rooted at $k$-tuples with backbone GNN as powerful as $m$-WL, the computational cost can be $O(|V|^{k}|V_s|^{m})$, where $V_s$ is the largest number of nodes among all subgraphs. We point out that $|V_s|$ is usually much smaller than $|V|$. For example, when counting cliques and stars, the hop parameter can be set to 1; when counting cycles and paths, the hop parameter can be set to $m/2$. Therefore subgraph GNNs are much more efficient. In conclusion, subgraph GNNs rooted at $k$-tuples with backbone GNN as powerful as $m$-WL can reach a similar counting power to $(m+k)$-WL while being much more efficient. This can be a key motivation for the use of subgraph GNNs in counting substructures.

\begin{figure*}[btp]
\begin{center}
\centerline{\includegraphics[width=1.9\columnwidth]{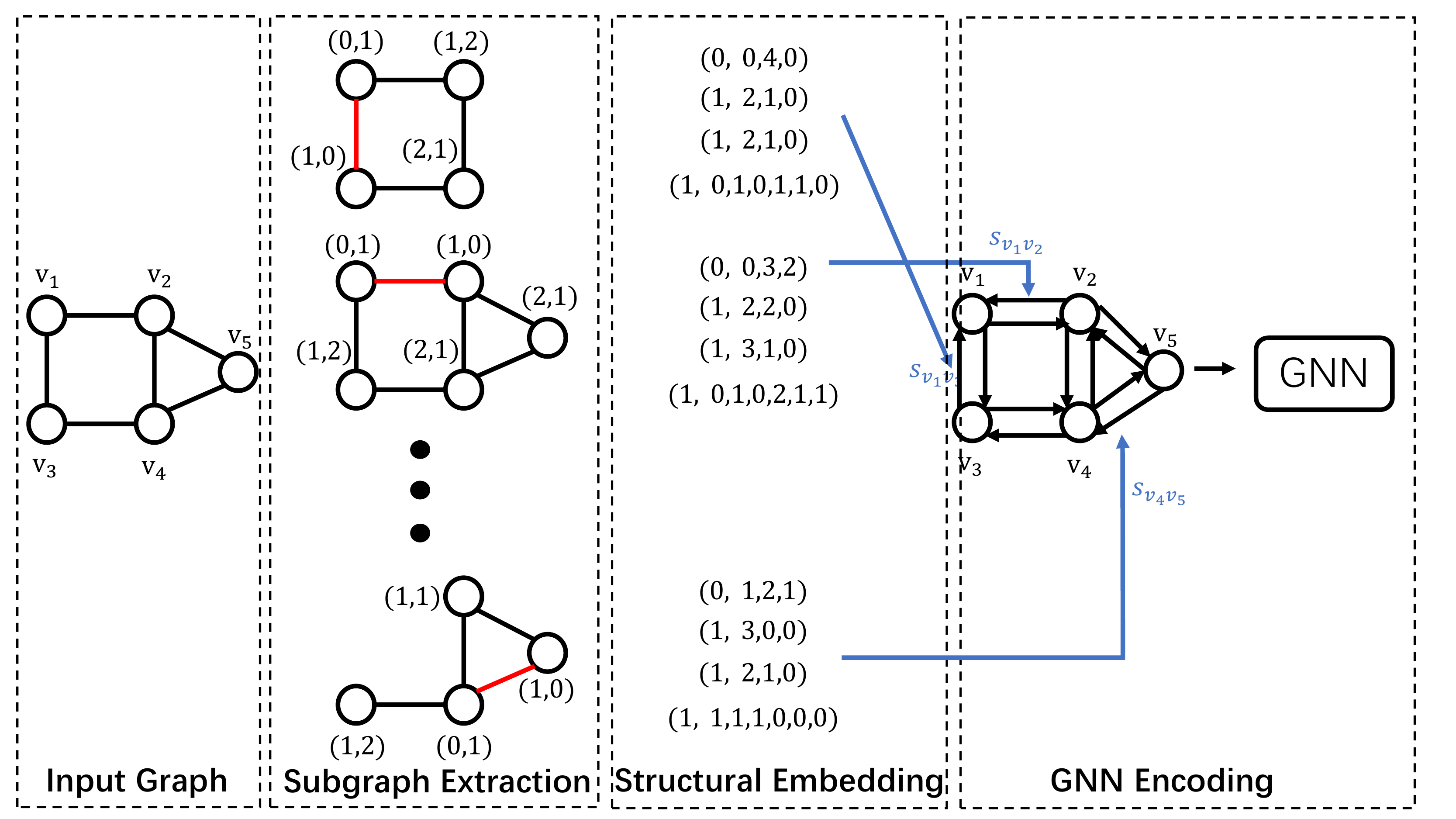}}
\caption{Framework of ESC-GNN. The rooted 2-tuples are colored red during subgraph extraction.}
\label{fig:ESC-GNN}
\end{center}
\vskip -0.3in
\end{figure*}

\section{Efficient Substructure Counting GNN (ESC-GNN)}
\label{sec:ESC-GNN}

Despite the strong substructure counting ability, subgraph GNNs are still much slower than MPNNs since they need to run backbone GNNs on all subgraphs. In this section, we propose a model, Efficient Substructure Counting GNN (ESC-GNN), which can count substructures efficiently and effectively, while, more importantly, does not need to run GNN on subgraphs.
ESC-GNN encodes the distance information within subgraphs into structural embeddings of edges in a preprocessing step. After that, it only needs to \textbf{run the backbone GNN on the input graph once rather than over all subgraphs}. We first introduce ESC-GNN, and then show its representation power theoretically.

\subsection{Framework of ESC-GNN}
\label{subsec:frame}
\textbf{Basic framework.} In this section, we first introduce the framework of ESC-GNN, which is illustrated in Figure~\ref{fig:ESC-GNN}. In ESC-GNN, we adopt subgraphs rooted at 2-tuples, and use MPNN as the backbone GNN. Subgraph GNNs rooted at 2-tuples are more expressive than those rooted at nodes, but at the cost of even higher computational cost~\citep{huang2023boosting}. Thus, instead of running backbone GNN over subgraphs to extract subgraph features (like distances), we directly encode them into some carefully designed structural embeddings, which are used as additional edge features of the input graph. An MPNN is then applied to this augmented graph. 
Specifically, let $h_v^t$ be the node representation for $v\in V$ in the $t$-th iteration. The update function is given by:
\begin{equation}
\label{eq:ESC-GNN}
h_v^{t+1} = W_1^t(h_v^t, \sum_{u \in N(v)}W_2^t(h_u^t, h_v^t, e_{uv}, s_{uv}))
\end{equation}
where $s_{uv}$ is the structural embedding for edge $uv$. 

\textbf{Choice of structural embedding.} Recall the proof of Theorem~\ref{th:sub} and the statement in Section~\ref{sec:power} imply that the distance information within subgraphs is  key to boosting the counting power of subgraph GNNs. Therefore, we encode such distance information, as well as some degree information, which provides the information of isomorphic types, as follows. An example is shown in Figure~\ref{fig:ESC-GNN}. 

\begin{itemize}
\item The degree encoding: for each subgraph, we first compute the degree of all nodes within the subgraph, and then use the degree histogram as the encoding. For example, the degree histogram for the first subgraph (the subgraph rooted at edge $v_1v_3$) of Figure~\ref{fig:ESC-GNN} is $(0,0,4,0)$, since there are 4 nodes with degree 2 in the subgraph.

\item The node-level distance encoding: for the subgraph rooted at edge $uv$, we use the shortest path distance histograms to the rooted nodes as the distance encoding. Take the first subgraph of Figure~\ref{fig:ESC-GNN} as an example. The distance histogram for $v_1$ is $(1,2,1,0)$, since there are one node with distance $0$ ($v_1$), two nodes with distance $1$ ($v_2$ and $v_3$), and one node with distance $2$ ($v_4$) to node $v_1$. The same holds for the other rooted node $v_3$.

\item The edge-level distance encoding: for the subgraph rooted at edge $uv$, define the label for each node as its shortest path distances to all the rooted nodes, i.e., $\forall u_1, f(u_1) = (d(u_1, u), d(u_1,v))$. Then we can define the label of edges in the subgraph as the concatenation of the label of its two end nodes, e.g., for edge $u_1v_1$, its label is $f(u_1v_1) = (f(u_1),f(v_1))$. We then use the edge-level distance histogram as the distance encoding. For example, for the subgraphs shown in Figure~\ref{fig:ESC-GNN}, there are seven types of edges: (0,1,1,0), (0,1,1,1), (0,1,1,2), (1,0,1,1), (1,0,2,1), (1,2,2,1), (2,1,2,1). Therefore the edge-level distance encoding for the first subgraph is (1,0,1,0,1,1,0), since there are one edge ($v_1v_3$) with label (0,1,1,0), one edge ($v_1v_2$) with label (0,1,1,2), one edge ($v_3v_4$) with label (1,0,2,1), and one edge $(v_2v_4)$ with label (1,2,2,1).
\end{itemize}

Finally, we concatenate the three encodings to get the final structural embedding $s_{uv}$. 

\textbf{Analysis on the structural embedding.} In terms of representation power, we find that (1) all these distance encodings can be extracted using an MPNN within the subgraph, (2) subgraph MPNNs rooted at 2-tuple can count structures such as 5-cycles~\citep{huang2023boosting}, while ESC-GNN cannot. Consequently, based on these considerations, we arrive at the following conclusion:

\begin{proposition}
\label{rm:esc}
ESC-GNN is less powerful than subgraph GNNs rooted at 2-tuples with MPNN as the backbone GNN.
\end{proposition}

In terms of algorithm efficiency, we can precompute these structural embeddings during subgraph extraction (the preprocessing cost can be amortized into each epoch/iteration), and ESC-GNN only needs to run the backbone GNN on the input graph. Therefore in every iteration, its computational cost is only $O(|E|)$, and its memory cost is $O(|V|)$, both the same as MPNN. As for subgraph MPNNs rooted at edges, they need to run the process of subgraph extraction and run the backbone GNN over all subgraphs. In every iteration, their computational cost is $O(|E||E_s'|)$, and their memory cost is $O(|E||V_s'|)$, where $|V_s'|$ and $|E_s'|$ are the average numbers of nodes and edges among all subgraphs. Even if considering subgraph MPNNs rooted at nodes, in every iteration, their computational cost is $O(|V||E_s|)$, and their memory cost is $O(|V||V_s|)$. Note that $|V||E_s| >> \frac{|V|D}{2} = |E|$, where $D$ is the average node degree. Therefore we can safely conclude that ESC-GNN is more efficient than subgraph GNNs. We will empirically evaluate its efficiency in the experiment.

\subsection{Representation Power of ESC-GNN}
\label{subsec:escgnn}
In this section, we analyze the representation power of ESC-GNN from two perspectives: its counting power and its ability to distinguish non-isomorphic graphs.

\textbf{Counting power of ESC-GNN.} Existing works~\citep{you2021identity,huang2023boosting} mainly focus on subgraph counting. Here we provide results on both subgraph counting and induced subgraph counting. We use four popular types of substructures: cycles, cliques, stars, and paths, as examples to show the counting power of ESC-GNN. The proof is provided in the appendix.

\begin{theorem}
\label{th:count}
In terms of subgraph counting, ESC-GNN can count (1) up to 4-cycles; (2) up to 4-cliques; (3) stars with arbitrary sizes; (4) up to 3-paths.
\end{theorem}

\begin{theorem}
\label{th:count_ind}
In terms of induced subgraph counting, ESC-GNN can count (1) up to 4-cycles; (2) up to 4-cliques; (3) up to 4-stars; (4) up to 3-paths.
\end{theorem}

\textbf{Compared with subgraph GNNs.} As shown in Proposition~\ref{rm:esc}, ESC-GNN is less powerful than subgraph MPNNs rooted at 2-tuples~\citep{huang2023boosting}. As for subgraph MPNNs rooted at nodes~\citep{zhang2021nested,you2021identity}, they can only count up to 4-cycles, 3-cliques, and 3-paths~\citep{huang2023boosting}. Therefore, ESC-GNN is more powerful than subgraph MPNNs rooted at nodes in terms of counting these substructures.

\textbf{From graph-level substructure counting to node-level substructure counting.} It is important to recognize that counting substructures at the graph level presents a lesser challenge compared to node-level substructure counting. The reason behind this is straightforward: the total number of substructures within a graph can be computed through the weighted sum of the counts of substructures passing each node. Conversely, the inverse process is not inherently feasible. Although our proof predominantly concentrates on graph-level counting, its principles bear the potential for adaptation to node-level counting scenarios.

To illustrate this concept, consider a structure $s$.  Let $n_{uv}$ denotes the number of $s$ that passes the 2-tuple 
$(u, v)$
, and let $p$ denotes the number of automorphisms of the 2-tuple starting at node $u$. In other words, p quantifies the number of 2-tuples starting from $u$ that mirror the isomorphic type of 
$(u, v)$ within structure $s$. Consequently, the aggregate count of structure $s$ that passes node $u$ is $\sum_{v \in N(u)} n_{uv} / k$. Take the counting of 3-cycles $s = \{u, v, w\}$ that passes node $u$ as an example. $k$ is 2 since in $s$, both $(u,v)$ and $(u, w)$ have the same isomorphic type as $(u, v)$ ($(u, w)$, resp.). Considering that there is a cycle that passes $(u, v)$
 ($(u, w)$, resp.), the number of 3-cycles that passes $u$ is $(1+1)/2=1$.

\textbf{The ability to distinguish non-isomorphic graphs.} In terms of distinguishing non-isomorphic graphs, we present the result below:

\begin{theorem}
\label{th:iso}
ESC-GNN is strictly more powerful than 2-WL, while not less expressive than 3-WL.
\end{theorem}

\textbf{Compared with subgraph GNNs.} In terms of distinguishing non-isomorphic graphs, subgraph MPNNs rooted at nodes~\citep{zhang2021nested,you2021identity} are strictly less powerful than 3-WL~\citep{frascaunderstanding2022}, while ESC-GNN is not less expressive than 3-WL. 

In addition, although able to distinguish most pairs of non-isomorphic graphs~\citep{babai1980random}, 2-WL fails to distinguish any pairs of non-isomorphic $r$-regular graphs with equal size. In this paper, we prove that ESC-GNN can distinguish almost all $r$-regular graphs:

\begin{theorem}
\label{th:reg}
Consider all pairs of $r$-regular graphs with $n$ nodes, let $3\leq r < (2log2n)^{1/2}$ and $\epsilon$ be a fixed constant. With the hop parameter $h$ set to $\lfloor(1/2 + \epsilon)\frac{log2n}{log(r-1)}\rfloor$, there exists an ESC-GNN that can distinguish $1-o(n^{-1/2})$ such pairs of graphs.
\end{theorem}

Furthermore, we also show that ESC-GNN's representation power also benefits from its global message passing layer in Section~\ref{sec:th_mp} in the appendix.

\section{Experiment}
To thoroughly analyze the property of ESC-GNN, we evaluate it from the following perspectives: (1) we evaluate its representation power in Section~\ref{subsec:rep}, to show whether it can reach the theoretical power shown in Section~\ref{subsec:escgnn}; (2) we evaluate its performance on real-world benchmarks in Section~\ref{subsec:real}, to show whether the increased representation power can boost its performance on real-world tasks; (3) we evaluate its efficiency in Section~\ref{subsec:eff}. 
The code is available at \url{https://github.com/pkuyzy/ESC-GNN}.
 
We provide the experimental details, the analysis on the limitation of the paper, and the information of the used assets in the appendix.

\textbf{Baselines.} We compare with baseline methods including (1) basic MPNNs~\citep{xu2018powerful, KipfW17}; (2) subgraph GNNs including NGNN~\citep{zhang2021nested}, IDGNN~\citep{you2021identity}, GIN-AK+~\citep{zhao2022stars}, SUN~\citep{frascaunderstanding2022}, DSS-GNN~\citep{bevilacqua2022equivariant}, OSAN~\citep{qianordered2022}, I$^2$-GNN~\citep{huang2023boosting}, and PL-GNN~\citep{bevilacqua2024efficient}; (3) high-order GNN models including 1-2-3-GNN~\citep{morris2019weisfeiler} and PPGN~\citep{maron2019provably}; (4) graph transformers including Graphormer-GD~\citep{zhang2023rethinking} and GraphGPS~\citep{rampavsek2022recipe}; and (5) GNNs that incorporate the substructure information, including LGP-GNN~\cite{barcelo2021graph} and GSN~\cite{bouritsas2022improving}.

\begin{table*}
	\centering
	\caption{Evaluation on Counting Substructures (norm MAE), cells with MAE less than 0.01 (an indicator of successful counting~\citep{huang2023boosting}) are colored yellow.} 
    \label{tab:count}
	\scalebox{0.9}{
	\begin{tabular}{lccccccccc}
	\hline\noalign{\smallskip}
	Dataset & Tailed Triangle & Chordal Cycle & 4-Clique & 4-Path & Triangle-Rectangle & 3-cycles & 4-cycles & 5-cycles & 6-cycles \\
	\noalign{\smallskip}\hline\noalign{\smallskip}
	MPNN & 0.3631 & 0.3114 & 0.1645 & 0.1592 & 0.2979 &  0.3515 & 0.2742 & 0.2088 & 0.1555\\
	ID-GNN & 0.1053 & 0.0454 & \cellcolor{yellow}0.0026 & 0.0273 & 0.0628 & \cellcolor{yellow}0.0006 & \cellcolor{yellow}0.0022 & 0.0490 & 0.0495 \\
	NGNN & 0.1044 & 0.0392 & \cellcolor{yellow}0.0045 & 0.0244 & 0.0729 & \cellcolor{yellow}0.0003 & \cellcolor{yellow}0.0013 & 0.0402 & 0.0439\\
	GIN-AK+ & \cellcolor{yellow}0.0043 & 0.0112 &\cellcolor{yellow} 0.0049 & \cellcolor{yellow}0.0075 & 0.1311 & \cellcolor{yellow}0.0004 & \cellcolor{yellow}0.0041 & 0.0133 & 0.0238\\
	PPGN & \cellcolor{yellow}0.0026 & \cellcolor{yellow}0.0015 & 0.1646 & \cellcolor{yellow}0.0041 & 0.0144 &\cellcolor{yellow}0.0005 &\cellcolor{yellow}0.0013 & \cellcolor{yellow}0.0044& \cellcolor{yellow}0.0079\\ 
	I$^2$-GNN & \cellcolor{yellow}0.0011 & \cellcolor{yellow}0.0010 & \cellcolor{yellow}0.0003 & \cellcolor{yellow}0.0041 & 0.0013 &\cellcolor{yellow} 0.0003 & \cellcolor{yellow}0.0016 & \cellcolor{yellow}0.0028 & \cellcolor{yellow}0.0082\\ 
    Graphormer-GD & 0.3660 & 0.2611 & 0.1580 & 0.1125 & 0.2460 & 0.3080 & 0.2317 & 0.1540 & 0.1380 \\
    GraphGPS & 0.0132 & 0.0630 & 0.1156 & 0.0910 & 0.0551 & 0.0882 & 0.1645 & 0.0462 & 0.1193\\
    \noalign{\smallskip}\hline\noalign{\smallskip}
	ESC-GNN & \cellcolor{yellow}0.0052 & 0.0169 & \cellcolor{yellow}0.0064 & 0.0254 & 0.0178& \cellcolor{yellow}0.0074 &\cellcolor{yellow}0.0044 & 0.0356& 0.0337\\
	\hline\noalign{\smallskip}
    \end{tabular}}
\end{table*}

\subsection{Representation Power of ESC-GNN}
\label{subsec:rep}

\textbf{Datasets.} We evaluate the representation power of ESC-GNN from two perspectives: 

(a) Its ability to differentiate non-isomorphic graphs. We use (1) EXP~\citep{abboud2021surprising}, which contains 600 pairs of non-isomorphic graphs that 1-WL/2-WL fails to distinguish; (2) SR25~\citep{balcilar2021breaking}, which contains 150 pairs of non-isomorphic strongly regular graphs that cannot be differentiated by 3-WL; (3) CSL~\citep{murphy2019relational}, which contains 150 regular graphs that 1-WL/2-WL fails to distinguish. These graphs are classified into 10 isomorphism classes. Classification accuracy is adopted as the evaluation metric.

(b) Its counting ability. We 
use the synthetic dataset from~\citep{zhao2022stars}. The task is to predict the number of substructures that pass each node in the given graph. The Mean Absolute Error (MAE) is adopted as the evaluation metric.

\textbf{Results.} In Table~\ref{tab:exp}, ESC-GNN achieves 100\% accuracy on all datasets. Considering that models as powerful as 3-WL (PPGN and 3-GNN) fail the SR25 dataset, the results serve as empirical evidence that ESC-GNN can effectively differentiate regular graphs (Theorem~\ref{th:reg}), and not less powerful than 3-WL (Theorem~\ref{th:iso}). In Table~\ref{tab:count}, ESC-GNN reaches less-than-0.01 MAE in terms of counting tailed triangles, 4-cliques, 3-cycles, and 4-cycles. The low error allows us to simply apply a rounding function to obtain the ground truth integer counting results, thus are considered an indicator of successful counting as in previous works~\citep{huang2023boosting}. These counting results exactly match Theorem~\ref{th:count}.
Generally speaking, ESC-GNN performs much better than MPNNs, and slightly beats or performs comparably with node-based subgraph GNNs such as ID-GNN, NGNN and GIN-AK+. Also, it performs inferior to subgraph GNNs rooted at edges (I$^2$-GNN). This serves as the empirical evidence for Proposition~\ref{rm:esc}. In addition, graph transformers such as GraphGPS~\footnote{We exclude the laplacian-based structural embeddings since they are not permutation-invariant, i.e., they may produce different outputs for the same graph.} and Graphormer-GD have shown  impressive performance in previous studies, particularly on graph-level prediction tasks. However, they perform inferior to subgraph GNNs and our proposed model on the substructure-counting task. This observation suggests potential limitations in the applicability of these transformer-based models to tasks involving substructure analysis.

\begin{table}
\centering
\caption{Test Accuracy on EXP/SR25/CSL}
\begin{tabular}{lcccccc}
    \noalign{\smallskip}\hline\noalign{\smallskip}
	Dataset & EXP & SR25 & CSL \\
	\noalign{\smallskip}\hline\noalign{\smallskip}
	MPNN & 50 & 6.67 & 10 \\
	NGNN  & 100 & 6.67 & - \\
	GIN-AK+ & 100 & 6.67 & -\\
	PPGN & 100 & 6.67 & -\\
	3-GNN & 99.7 & 6.67 & 95.7 \\
	I$^2$-GNN & 100 & 100 & 100 \\
	\noalign{\smallskip}\hline\noalign{\smallskip}
	ESC-GNN & 100 & 100 & 100 \\
\noalign{\smallskip}\hline\noalign{\smallskip}
    \label{tab:exp}
    
	\end{tabular}
\end{table} 

\begin{table}
\centering
\caption{Evaluation on Algorithm Efficiency on OGBG-HIV (Seconds).}
\scalebox{0.9}{
\begin{tabular}{lcccc}
\hline\noalign{\smallskip}
Model & Pre & Run \\
\noalign{\smallskip}\hline\noalign{\smallskip}
MPNN & 2.7 & 6296.8 \\
NGNN & 1288.0 & 14862.9 \\
I$^2$-GNN & 2806.5& 1042963.7  \\
GIN-AK+(Sample) & 376.2 & 10275 \\
OSAN (Sample) & 6.6 & 7980.1 \\
\noalign{\smallskip}\hline\noalign{\smallskip}
ESC-GNN & 1782.5 & 6301.0 \\
\noalign{\smallskip}\hline\noalign{\smallskip}
\label{tab:eff}
\end{tabular}}
\end{table}

\begin{table}
\centering
\caption{Evaluation on QM9 (MAE)} 
\label{tab:qm9}
\scalebox{1.0}{
\begin{tabular}{lcccccccc}
\hline\noalign{\smallskip}
Dataset & 1-GNN & 1-2-3-GNN & NGNN & I$^2$-GNN & ESC-GNN \\
\noalign{\smallskip}\hline\noalign{\smallskip}
$\mu$ & 0.493 & 0.476 & 0.428 & 0.428 & \textbf{0.231}\\
$\alpha$ & 0.78 & 0.27 & 0.29 & \textbf{0.230} & 0.265\\
$\epsilon_{\text{homo}}$  &0.00321 &0.00337& 0.00265 &0.00261& \textbf{0.00221}\\
$\epsilon_{\text{lumo}}$  &0.00355& 0.00351  &0.00297& 0.00267&\textbf{0.00204}\\
$\Delta_{\epsilon}$ & 0.0049 &0.0048 & 0.0038& 0.0038&\textbf{0.0032}\\
$R^2$ & 34.1& 22.9 &20.5 &18.64&\textbf{7.28}\\
ZPVE & 0.00124& 0.00019& 0.0002 &\textbf{0.00014}&0.00033\\
$U_0$ & 2.32& 0.0427 &0.295& \textbf{0.211}&0.645\\
$U$ & 2.08& \textbf{0.111}& 0.361 &0.206&0.380\\
$H$ & 2.23 &\textbf{0.0419} &0.305 &0.269&0.427\\
$G$ & 1.94 &\textbf{0.0469} & 0.489 &0.261&0.384\\
$C_v$ &0.27& 0.0944& 0.174 &\textbf{0.0730} &0.105\\
\noalign{\smallskip}\hline\noalign{\smallskip}
\end{tabular}}
\end{table}

\subsection{Real-World Tasks}
\label{subsec:real}
We present our experiments conducted on datasets including QM9~\citep{ramakrishnan2014quantum, wu2018moleculenet}, OGBG-MOLHIV~\cite{hu2020open}, and ZINC~\cite{dwivedi2020benchmarking}. These results are reported in Table~\ref{tab:qm9} and Table~\ref{tab:mol}. On QM9 and ZINC, the task is to perform regression on various graph properties. While on OGBG-MOLHIV, the task is to predict the classification of molecule graphs. Detailed statistics are available in Section~\ref{sec:exp_detail} in the appendix.

\begin{table*}
\centering
\caption{Evalation on ZINC and OGB datasets.} 
\label{tab:mol}
\scalebox{1.0}{
\begin{tabular}{lccc}
\hline\noalign{\smallskip}
Dataset & OGBG-HIV (AUCROC) &  ZINC-250K (MAE) & ZINC-12K (MAE) \\
\noalign{\smallskip}\hline\noalign{\smallskip}
GIN~\citep{xu2018powerful} & 77.07$\pm$1.49 & 0.088$\pm$0.002 & 0.163$\pm$0.004 \\
PNA~\citep{corso2020principal} & 79.05$\pm$1.32 & - & 0.188$\pm$0.004\\
GSN~\citep{bouritsas2022improving} & 77.99$\pm$0.01 & - & 0.115$\pm$0.012\\
LGP-GNN~\cite{barcelo2021graph} & - & - & 0.135$\pm$0.010\\
DGN~\citep{beaini2021directional} & 79.70$\pm$0.97 & - & 0.168$\pm$0.003 \\
CIN~\citep{bodnar2021weisfeiler1} & \textbf{80.94$\pm$0.57} & 0.022$\pm$0.002 & 0.079$\pm$0.006 \\
NGNN~\citep{zhang2021nested} & 78.34$\pm$1.86 & 0.029$\pm$0.001 & 0.111$\pm$0.003\\
GIN-AK+~\citep{zhao2022stars} & 79.61$\pm$1.19 & - & 0.080$\pm$0.001\\
OSAN~\citep{qianordered2022} & 78.01$\pm$1.04 & - & 0.126$\pm$0.006\\ 
SUN~\citep{frascaunderstanding2022} & 80.03$\pm$0.55 &  - & 0.083$\pm$0.003 \\
DSS-GNN~\citep{bevilacqua2022equivariant} & 76.78$\pm$1.66 & - & 0.102$\pm$0.003 \\
I$^2$-GNN~\citep{huang2023boosting} & 78.68$\pm$0.93 & 0.023$\pm$0.001 & 0.083$\pm$0.001 \\
PL-GNN~\citep{bevilacqua2024efficient} & 78.49$\pm$1.01 & - & 0.109$\pm$0.005 \\
Graph Transformer~\citep{rampavsek2022recipe} & 77.40$\pm$1.77 & 0.030$\pm$0.006 & 0.113$\pm$0.003 \\
\noalign{\smallskip}\hline\noalign{\smallskip}
ESC-GNN & 
79.21$\pm$0.84
& 
\textbf{0.021$\pm$0.003}
& \textbf{0.075$\pm$0.002}\\
\noalign{\smallskip}\hline\noalign{\smallskip}
\end{tabular}}
\end{table*}

\begin{table*}
	\centering
	\caption{Ablation study on the proposed structural embedding (norm MAE).} 
    \label{tab:count_ab}
	\scalebox{0.95}{
	\begin{tabular}{lccccccccc}
	\hline\noalign{\smallskip}
	Dataset & Tailed Triangle & Chordal Cycle & 4-Clique & 4-Path & Triangle-Rectangle & 3-cycles & 4-cycles & 5-cycles & 6-cycles \\
	\noalign{\smallskip}\hline\noalign{\smallskip}
	MPNN & 0.3631 & 0.3114 & 0.1645 & 0.1592 & 0.2979 &  0.3515 & 0.2742 & 0.2088 & 0.1555\\
 	\noalign{\smallskip}\hline\noalign{\smallskip}
	ESC-GNN & \cellcolor{yellow}0.0052 & 0.0169 & \cellcolor{yellow}0.0064 & 0.0254 & 0.0178 & \cellcolor{yellow}0.0074 &\cellcolor{yellow}0.0044 & 0.0356& 0.0337\\
    \noalign{\smallskip}\hline\noalign{\smallskip}
     (- degree) &0.0121 & 0.0492& 0.0106& 0.0322 & 0.0349 &0.0342 &0.0144 &0.0513 & 0.0652\\
    (- node-level dist) & 0.0382 & 0.0344 & 0.0222 &0.0428 & 0.0256 & 0.0157 & 0.0261 & 0.0492 &0.0608  \\
    (- edge-level dist) & 0.0208&0.2811 &0.0497 & 0.0584&0.185 & 0.2617&0.2244 & 0.1654& 0.1364\\
	\hline\noalign{\smallskip}
    \end{tabular}}
\end{table*}

\textbf{Comparison with backbone models.} On OGBG-MolHIV and QM9, we use GIN~\citep{xu2018powerful} as the backbone model. While on ZINC, we use \textbf{a plain graph transformer without positional encodings} from~\citep{rampavsek2022recipe} as the backbone model. The comparison between ESC-GNN and the backbone models demonstrates that the proposed structural information not only enhances the theoretical representation power but also improves the empirical representation power.

\textbf{Comparison with subgraph GNNs.} When compared to subgraph GNNs, such as PL-GNN, OSAN, DSS-GNN, I$^2$-GNN, NGNN, and GIN-AK+, ESC-GNN parallels or slightly surpasses these models on these benchmarks. 
This finding substantiates that the proposed structural embedding effectively captures valuable information from subgraph GNNs, thereby benefiting downstream tasks. Moreover, the framework of ESC-GNN is simple enough to avoid problems such as overfitting. However, it is noteworthy that ESC-GNN's performance on $U_0$, $U$, and $H$ in Table~\ref{tab:qm9}—representing Internal Energy at 0K, Internal Energy at 298.15K, and Enthalpy at 298.15K, respectively—does not reach optimal levels. To calculate these targets, computational methods take into account the interactions between all the atoms in the molecule and their surroundings, including any heat or work exchanged with the environment. As a result, globally expressive models (such as 1-2-3 GNN) exhibit superior performance, while subgraph GNNs, which utilize local information to enhance graph representation, show relatively diminished efficacy. 

In terms of OGBG-HIV, although ESC-GNN outperforms the majority of  subgraph GNNs, it performs inferior to the state-of-the-art (SOTA) models. This phenomenon could be attributed to the intricate nature of HIV replication inhibition, which is governed by a complex interplay of both local and global attributes. Specifically, local attributes include the structure and functionality of specific molecules and their enzymatic active sites, while global attributes encompass the overarching graph structure of both the viral components and the molecular constituents within cellular matrices and their interactive cross-sections.

Consequently, SOTA models combine global information — such as graph pooling models, high-order message passing techniques, or graph transformers — with local information. An example of the latter is the molecular fingerprint, which encapsulates the count of various specific molecules particular to specific domains. While these fingerprints offer detailed insights, there is a notable risk of overfitting when applied to certain datasets, potentially limiting their generalizability and applicability in broader contexts.

\textbf{Comparison with GNNs that incorporate the substructure information.} GSN~\citep{bouritsas2022improving} and LGP-GNN~\citep{barcelo2021graph} incorporate the number of specific substructures as augmented features. The approach heavily relies on domain-specific knowledge and the desired substructures may significantly vary across different tasks. In contrast, our ESC-GNN framework does not explicitly encode the number of substructures. Instead, it encodes the more general distance information of subgraphs, which enables plain GNNs to implicitly count many substructures without pre-selecting which to include. The general distance information contains more structural information than the number of specific substructures. For instance, we have proven that the number of 4-cycles can be computed using the distance information. Conversely, deducing distance information solely from the count of 4-cycles is impracticable, given the multiplicity of configurations that can form a 4-cycle. The superiority is also justified by empirical evidence in Section~\ref{sec:count_real} in the Appendix.

Regarding the real-world dataset in Table~\ref{tab:mol}, GSN achieves a MAE score of 0.115 on ZINC, whereas ESC-GNN attains a MAE score of 0.075. Similarly, GSN records an AUCROC score of 77.99 on OGBG-HIV, whereas ESC-GNN demonstrates a higher AUCROC score of 79.21. These results clearly highlight the enhanced representation power of our proposed model.

\subsection{Algorithm Efficiency}
\label{subsec:eff}

In terms of algorithm efficiency, we conduct a comparison between the ESC-GNN and five established baseline models: a baseline MPNN, NGNN, OSAN, and GIN-AK+ whose subgraphs are rooted at node, and I$^2$-GNN whose subgraphs are rooted at 2-tuples. We accelerate the code of I$^2$-GNN by implementing a parallel subgraph preprocessing strategy. We report the data preprocessing time and the standard running time for running 100 epochs on ogbg-hiv 
in Table~\ref{tab:eff}. 
It is crucial to note that while the sampling strategies employed in OSAN and GIN-AK+ contribute to increased speed, they lose the theoretical substructure counting ability. This limitation arises because substructures present in unsampled subgraphs remain uncounted, leading to a reduction in representation power and, consequently, diminished performance. For example, employing subgraph sampling in OSAN results in an AUC-ROC score of 75.96 on OGBG-HIV, which is significantly less than the 78.01 score reported in Table~\ref{tab:mol}. Therefore, a direct comparison with our approach, which does not rely on such sampling strategies, may not be entirely fair. 

Nevertheless, for the sake of a thorough analysis, we have included these comparative results in Table~\ref{tab:eff}. As shown in the table, ESC-GNN exhibits a notable advantage in terms of total running time compared to subgraph GNNs, including those utilizing theoretically unsound sampling methods. This efficiency is attributed to the fact that preprocessing in our model is a one-time requirement, with the processed data being reusable across all subsequent training and inference stages. This finding aligns with our observations presented in Section~\ref{subsec:frame}, further validating the superior efficiency and effectiveness of ESC-GNN in handling graph data.

\subsection{Ablation Study}
\label{subsec:ablation}
We evaluate the effectiveness of each part of the proposed structural embedding on the substructure counting dataset. For the proposed three types of structural embedding, we delete one of them from the original structural embedding every time and report the results in Table~\ref{tab:count_ab}. We observe that after removing the three types of embedding, ESC-GNN performs worse compared with its original version, especially after removing the edge-level distance encoding. This is consistent with our theoretical results, where in Theorem~\ref{th:sub}, we show that the proposed structural embedding contains key information for the counting power of subgraph GNNs; in Theorem~\ref{th:count} and Theorem~\ref{th:count_ind}, we show that the edge-level distance information encodes the information of certain types of substructures.

\section{Conclusion}

The huge computational cost is associated with subgraph GNNs due to the requirement of running backbone GNNs among all subgraphs. To address this challenge and enable efficient substructure counting with GNNs, we theoretically show that the distance information within subgraphs is key to boosting the counting power of GNNs. We then encode this information into a structural embedding and enhance standard GNN models with this embedding, eliminating the need to learn representations over all subgraphs. Experiments on various benchmarks demonstrate that the proposed model retains the representation power of subgraph GNNs while running much faster. It can potentially enhance the utility of subgraph GNNs in a variety of applications that require efficient substructure counting.


\section*{Acknowledgement}
The work of Zuoyu Yan, Liangcai Gao, and Zhi Tang is supported by the projects of National Science and Technology Major Project (2021ZD0113301) and National Natural Science Foundation of China (No. 62376012), which is also a research achievement of Key Laboratory of Science, Technology and Standard in Press Industry (Key Laboratory of Intelligent Press Media Technology). The work of Junru Zhou and Muhan Zhang is supported by the National Natural Science Foundation of China (62276003), the National Key R\&D Program of China (2021ZD0114702), and Alibaba Innovative Research Program.

\nocite{langley00}

\bibliography{example_paper}
\bibliographystyle{ACM-Reference-Format}

\citestyle{acmauthoryear}
\newpage
\appendix
\onecolumn

\section{The Weisfeiler-Leman algorithm and Message Passing Neural Network}
\textbf{$1$-WL.} We first describe the classic Weisfeiler-Leman algorithm (1-WL). For a given graph $G$, 1-WL aims to compute the coloring for each node in $V$. It computes the node coloring for each node by aggregating the color information from its neighborhood iteratively. In the 0-th iteration, the color for node $v \in V$ is $c_v^0$, denoting its initial isomorphic type\footnote{The term "isomorphic type" is based on previous work~\citep{morris2019weisfeiler}, which gives each $k$-tuple of nodes an initial feature such that two $k$-tuples receive the same initial feature if and only if their induced subgraphs (indexed by node order in the $k$-tuple) are isomorphic.}. For labeled graphs, the isomorphic type of a node is simply its node feature. For unlabeled graphs, we give the same 1 to all nodes. In the $t$-th iteration, the coloring for $v$ is computed as $$c_v^t = \text{HASH}(c_v^{t-1}, \{\!\!\{c_u^{t-1}:u \in N(v)\}\!\!\})$$,
where HASH is a bijective hashing function that maps the input to a specific color. The process ends when the colors for all nodes between two iterations are unchanged. If two graphs have different coloring histograms in the end (e.g., different numbers of nodes with the same color), then $1$-WL detects them as non-isomorphic.

\textbf{$k$-WL.} For each $k\geq 2, k \in \mathbb{N}$, the $k$-dimensional WL algorithm ($k$-WL) colors $k$-tuples instead of nodes. In the 0-th iteration, the isomorphic type of a $k$-tuple $\vec{v}$ is given by the hashing of 1) the tuple of colors associated with the nodes of the $\vec{v}$, and 2) the adjacency matrix of the subgraph induced by $\vec{v}$ ordered by the node ordering within $\vec{v}$. 
In the $t$-th iteration, its coloring is updated by: 
$$c_{\vec{v}}^t = \text{HASH}(c_{\vec{v}}^{t-1}, c_{\vec{v}, (1)}^{\textcolor{black}{t}},...,c_{\vec{v}, (k)}^{\textcolor{black}{t}})$$,
where $$c_{\vec{v}, (i)}^t = \{\!\!\{c_{\vec{u}}^{t-1}|\vec{u} \in N_i(\vec{v})\}\!\!\}, i \in [k]$$. 

Here, $N_i(\vec{v}) = \{(v_1,...,v_{i-1}, w, v_{i+1},...,v_k) | w \in V\}$ is the $i$-th neighborhood of $\vec{v}$. Intuitively, $N_i(\vec{v})$ is obtained by replacing the $i$-th component of $\vec{v}$ by each node from $V$. Besides the updating function, other procedures of $k$-WL are analogous to 1-WL. In terms of distinguishing graphs, 2-WL is as powerful as 1-WL, and for $k\geq 2$, $(k+1)$-WL is strictly more powerful than $k$-WL.

\textbf{MPNNs.} MPNNs are a class of GNNs that learns node representations by iteratively 
aggregating messages from neighboring nodes. Let $h_v^t$ be the node representation for $v \in V$ in the $t$-th iteration. It is usually initialized with the node's intrinsic attributes. In the $(t+1)$-th iteration, it is updated by $$h_v^{t+1} = W_1^t(h_v^t, \sum_{u \in N(v)}W_2^t(h_u^t, h_v^t, e_{uv}))$$,

where $W_1^t$ and $W_2^t$ are two learnable functions. MPNN's power in terms of distinguishing non-isomorphic graphs is upper bounded by 2-WL~\citep{xu2018powerful}.

\section{The Proof of Theorem~\ref{th:disab}}

We begin by introducing the \textit{graph isomorphism}. For a pair of graphs $G_1 = (V_1, E_1)$ and $G_2 = (V_2, E_2)$, if there exists a bijective mapping $f: V_1 \rightarrow V_2$, so that for any edge $(u_1, v_1) \in E_1$, it satisfies that $(f(u_1), f(v_1)) = (u_2, v_2) \in E_2$, then $G_1$ is isomorphic to $G_2$, otherwise they are not isomorphic. Up to now, there is no polynomial algorithm for solving the graph isomorphism problem. One popular method is to use the $k$-order Weisfeiler-Leman~\citep{weisfeiler1968reduction} algorithm ($k$-WL). It is known that  1-WL is as powerful as 2-WL, and for $k \geq 2$, $(k+1)$-WL is more powerful than $k$-WL.

We then prove that $k$-WL can't count all connected substructures with $(k+1)$ nodes (specifically, $(k+1)$-cliques). We restate the result as follows: 

\begin{theorem}
For any $k\geq 2$, there exists a pair of graphs $G$ and $H$, such that $G$ contains a $(k+1)$-clique as its subgraph while $H$ does not, and that $k$-WL can't distinguish $G$ from $H$.
\end{theorem}

\begin{proof}
	The counter-example is inspired by the well-known Cai-F\"urer-Immerman (CFI) graphs \citep{cai1992optimal}. We define a sequence of graphs $G_k^{(\ell)}, \ell = 0, 1, \ldots, k+1$ as following,
	\begin{equation}
		\begin{split}
			V_{G_k^{(\ell)}} =& \Big\{u_{a, \vec{v}}\Big| a\in[k+1], \vec{v}\in\{0,1\}^k \text{ and }\\
			&\quad\begin{array}{ll}
				\vec{v} \text{ contains an even number of } 1 \text{'s}, & \text{if }a=1,2,\ldots, k-\ell+1, \\
				\vec{v} \text{ contains an odd number of } 1 \text{'s},  & \text{if }a=k-\ell+2,\ldots, k+1.
			\end{array}\Big\}
		\end{split}
	\end{equation}
	Two nodes $u_{a,\vec{v}}$ and $u_{a',\vec{v}'}$ of $G_k^{(\ell)}$ are connected iff there exists $m\in [k]$ such that $a' \mod (k+1) = (a+m) \mod (k+1)$ and $v_m = v'_{k-m+1}$. We have the following lemma.

	\begin{lemma}
		\emph{(a)} For each $\ell = 0, 1, \ldots, k+1$, $G_k^{(\ell)}$ is an undirected graph with $(k+1)2^{k-1}$ nodes;

		\emph{(b)} The set of graphs $G_k^{(\ell)}$ with an odd $\ell$ are mutually isomorphic; similarly, the set of graphs $G_k^{(\ell)}$ with an even $\ell$ are mutually isomorphic.
	\end{lemma}

	It's easy to verify (a). To prove (b), it suffices to prove $G_k^{(\ell)}$ is isomorphic to $G_k^{(\ell+2)}$ for all $\ell = 0,1,\ldots, k-1$. We apply a \emph{renaming} to the nodes of $G_k^{(\ell)}$: we flip the $1^\mathrm{st}$ bit of $\vec{v}$ in every node named $u_{k-\ell, \vec{v}}$, and flip the $k^\mathrm{th}$ bit of $\vec{v}$ in every node named $u_{k-\ell+1, \vec{v}}$. Since this is a mere renaming of nodes, the resulting graph is isomorphic to $G_k^{(\ell)}$. However, it's also easy to see that the resulting graph follows the construction of $G_k^{(\ell+2)}$. Therefore, we assert that $G_k^{(\ell)}$ must be isomorphic to $G_k^{(\ell+2)}$.

	Now, let's ask $G=G_k^{(0)}$ and $H=G_k^{(1)}$. Obviously there is a $(k+1)$-clique in $G$: nodes $u_{j,0^k}, j=1,2, \ldots, k+1$ are mutually adjacent by definition of $G_k^{(0)}$. On the contrary, we have
	\begin{lemma}
		There's no $(k+1)$-clique in $H$.
	\end{lemma}

	The proof is given below. Assume there is a $(k+1)$-clique in $H$. Since there's no edge between nodes $u_{a, \vec{v}}$ with an identical $a$, the $(k+1)$-clique must contain exactly one node from every node set $\{u_{a,\vec{v}}\}$ for each fixed $a\in[k+1]$. We further assume that the $(k+1)$ nodes are $u_{a, b_{a1}b_{a2}\ldots b_{ak}}, a=1,2,\ldots, k+1$. Using the condition for adjacency, we have
	\begin{align}
		       & b_{2k} = b_{11},                                                               \\
		       & b_{3k} = b_{21}, b_{3(k-1)} = b_{12},                                          \\
		       & b_{4k} = b_{31}, b_{4(k-1)} = b_{22}, b_{4(k-2)} = b_{13},                     \\
		\notag & \cdots\cdots\cdots\cdots                                                       \\
		       & b_{(k+1)k} = b_{k1}, b_{(k+1)(k-1)} = b_{(k-1)2}, \ldots, b_{(k+1)1} = b_{1k}.
	\end{align}
	Applying the above identities to the summation
	\begin{align}
		\label{contradiction}
		\sum_{a=1}^{k+1}\sum_{j=1}^k b_{aj} = 2\sum_{j=1}^k\left(b_{1j}+b_{2j}+\cdots+b_{(k-j+1)j}\right),
	\end{align}
	we see that it should be even. However, by definition of $G_k^{(1)}$, there are an even number of $1$'s in $b_{a1}b_{a2}\ldots b_{ak}$ when $a\in [k]$, and an odd number of $1$'s when $a=k+1$. Therefore, the sum in \eqref{contradiction} should be odd. This leads to a contradiction.

	Finally, to prove the $k$-WL equivalence of $G$ and $H$, we have
	\begin{lemma}
		$k$-WL can't distinguish $G$ and $H$.
	\end{lemma}

	By virtue of the equivalence between $k$-WL and pebble games~\citep{grohe_otto_2015}, it suffices to prove that Player II will win the $\mathcal{C}_{k}$ bijective pebble game on $G$ and $H$. We state the winning strategy for Player II as following. Since $G$ and $H$ are isomorphic with nodes $\{u_{k+1,*}\}$ deleted, Player II can always choose an isomorphism $f:G-\{u_{k+1,*}\}\rightarrow H-\{u_{k+1,*}\}$ to survive if Player I never places a pebble on nodes $u_{k+1,*}$. Furthermore, since $k$ pebbles can occupy nodes with at most $k$ different values of $a$ (in $u_{a,\vec{v}}$), there's always a set of pebbleless nodes $\{u_{a_0,\vec{v}}\}$ with some $a_0\in[k+1]$. Therefore, Player II only needs to do proper renaming on $H$ between $u_{k+1,*}$ and $u_{a_0,*}$ as stated above. This makes every $\vec{v}$ in $u_{a_0,\vec{v}}$ have an odd number of $1$'s. Player II then chooses the isomorphism on $G-\{u_{a_0,*}\}$ and $H^{\mathrm{renamed}}-\{u_{a_0,*}\}$. This way, Player II never loses since there are not enough pebbles for Player I to make use of the oddity at the currently pebbleless set of nodes.

\end{proof}

\begin{remark}
    Notice that if we take $k=2$, then $G$ is two 3-cycles while $H$ is a 6-cycle, which 2-WL cannot differentiate; if we take $k=3$, then $G$ is the 4*4 Rook's graph while $H$ is the Shrikhande graph, which 3-WL cannot differentiate. In these special cases, the above construction complies with our well-known examples.
\end{remark}

\section{The Proof of Theorem~\ref{th:sub}}

We restate the theorem as follows:

\begin{theorem}
For any connected substructure with no more than $k+m$ ($m\geq 2, k > 0$) nodes, there exists a subgraph GNN rooted at $k$-tuples with backbone GNN as powerful as $m$-WL that can count it. 
\end{theorem}

\begin{proof}
Based on Remark~\ref{rm:de} in the main paper, there exists a type of connected $k$-tuple that satisfies the decomposition of the connected substructure. Then the substructure can be separated into 2 subgraphs: the nodes that belong to the $k$-tuple (we call them rooted nodes) and the nodes that do not belong to the $k$-tuple (we call them non-rooted nodes). Formally, for the given two substructures $G_1 = (V_1, E_1)$ and $G_2 = (V_2, E_2)$, we define the subgraph that formed by the rooted nodes of $G_1$ ($G_2$, resp.) as $G_{1,r} = (V_{1, r}, E_{1,r})$ ($G_{2,r} = (V_{2,r}, E_{2,r})$, resp.). Similarly, define the subgraph that formed by the non-rooted nodes of $G_1$ ($G_2$, resp.) as $G_{1,n} = (V_{1, n}, E_{1,n})$ ($G_{2,n} = (V_{2,n}, E_{2,n})$, resp.). We can also define the subgraph formed between the rooted nodes and non-rooted nodes of $G_1$ ($G_2$, resp.) as $G_{1,c} = (V_{1, c}, E_{1,c})$ ($G_{2,c} = (V_{2,c}, E_{2,c})$, resp.). Then it is easy to find that for $G_1$, $V_1 = (V_{1,r} \cup V_{1, n})$, $V_{1,c} \subseteq V_1$, $E_1 = E_{1,r} \cup E_{1, n} \cup E_{1, c}$. There is no intersection between $V_{1,r}$ and $V_{1, n}$, and there is no intersection between $E_{1,r}$, $E_{1,c}$, and $E_{1, n}$. The same holds for $G_2$.

If $G_1$ and $G_2$ are non-isomorphic, then there can be three potential situations: (1) $G_{1,r}$ and $G_{2,r}$ are non-isomorphic; (2) $G_{1,n}$ and $G_{2,n}$ are non-isomorphic; (3) $G_{1,c}$ and $G_{2,c}$ are non-isomorphic. In the following section, we will prove that in all these situations, the subgraph GNN can distinguish between $G_1$ and $G_2$.

\textbf{$G_{1,r}$ and $G_{2,r}$ are non-isomorphic.}  It denotes that for $G_1$ and $G_2$, the selected $k$-tuples are different. Then of course the subgraph GNN can differentiate $G_1$ and $G_2$.

\textbf{$G_{1,n}$ and $G_{2,n}$ are non-isomorphic.} Recall that we use a backbone GNN as powerful as $m$-WL to encode the information within the subgraph, and $V_{1,n}$ and $V_{2,n}$ contain nodes no more than $m$ nodes. Therefore if $G_{1,n}$ and $G_{2,n}$ are non-isomorphic, then they will have different isomorphic types, and thus can be distinguished by the backbone GNN.

\textbf{$G_{1,c}$ and $G_{2,c}$ are non-isomorphic.} We define the label of all nodes in $G_1$ (the same holds for $G_2$) as its distance to the nodes in the $k$-tuple. Formally, let $V_{1,r} = \{v_{1,r,1},...,v_{1,r,k}\}$, then $\forall u \in V_1$, its label $f_1(u) = (d(u, v_{1,r,1}),...,d({u,v_{1,r,k}}), I(u))$, where $d(.)$ denotes the shortest path distance between two nodes, and $I(u)$ denotes the label that reflects the isomorphic type of $u$ encoded by the subgraph GNN within $G_{1,n}$.

Since the substructures are connected, there exists at least a node $u_1 \in V_{1,n}$, whose label contains at least an index with the value ``1". For $f(u_1)$, the indices with value ``1" denotes that there exist edges between $u_1$ and the corresponding nodes in $V_{1,r}$. While the indices with value larger than 1 denote that there is no edge between $u_1$ and the corresponding nodes in $V_{1,r}$. The same holds for $u_2$ and $V_{2,r}$. Therefore, if the subgraph formed by $u_1$ and $V_{1,r}$ and the subgraph formed by $u_2$ and $V_{2,r}$ are non-isomorphic, then $f_1(u_1)$ and $f_2(u_2)$ are different, and the subgraph GNN can differentiate the two substructures. Also, if the $I(u_1)$ and $I_(u_2)$ are different, then the subgraph GNN can also distinguish $G_1$ and $G_2$. We can then consider the next nodes, and continue the process inductively.

Therefore, if for all nodes in $V_{1,n}$, we can find a unique node in $V_{2,n}$ that has the same label as it. Then $G_{1,c}$ and $G_{2,c}$ are isomorphic. Reversely, if $G_{1,c}$ and $G_{2,c}$ are non-isomorphic, then there exists at least a node in $V_{1,n}$, that we cannot find a unique node in $V_{2,n}$ that has the same label as it. Therefore the subgraph GNN can differentiate $G_1$ and $G_2$.

Then we can assign each substructure a unique color according to its isomorphic type, and use the color histogram of the graph as the output function. If two graphs have different numbers of certain substructures, then the color histogram will be different.

Based on the above results, the subgraph GNN can count the given type of connected substructure.

\end{proof}

\section{The Proof of Theorem~\ref{th:iso}}
We state the result as follows.

\begin{theorem}
ESC-GNN is strictly more powerful than 2-WL, while not less expressive than 3-WL.
\end{theorem}

\begin{figure}[btp]
	\centering
	\subfigure[]{
		\begin{minipage}[t]{0.43\linewidth}
			\centering
			\includegraphics[width=\columnwidth]{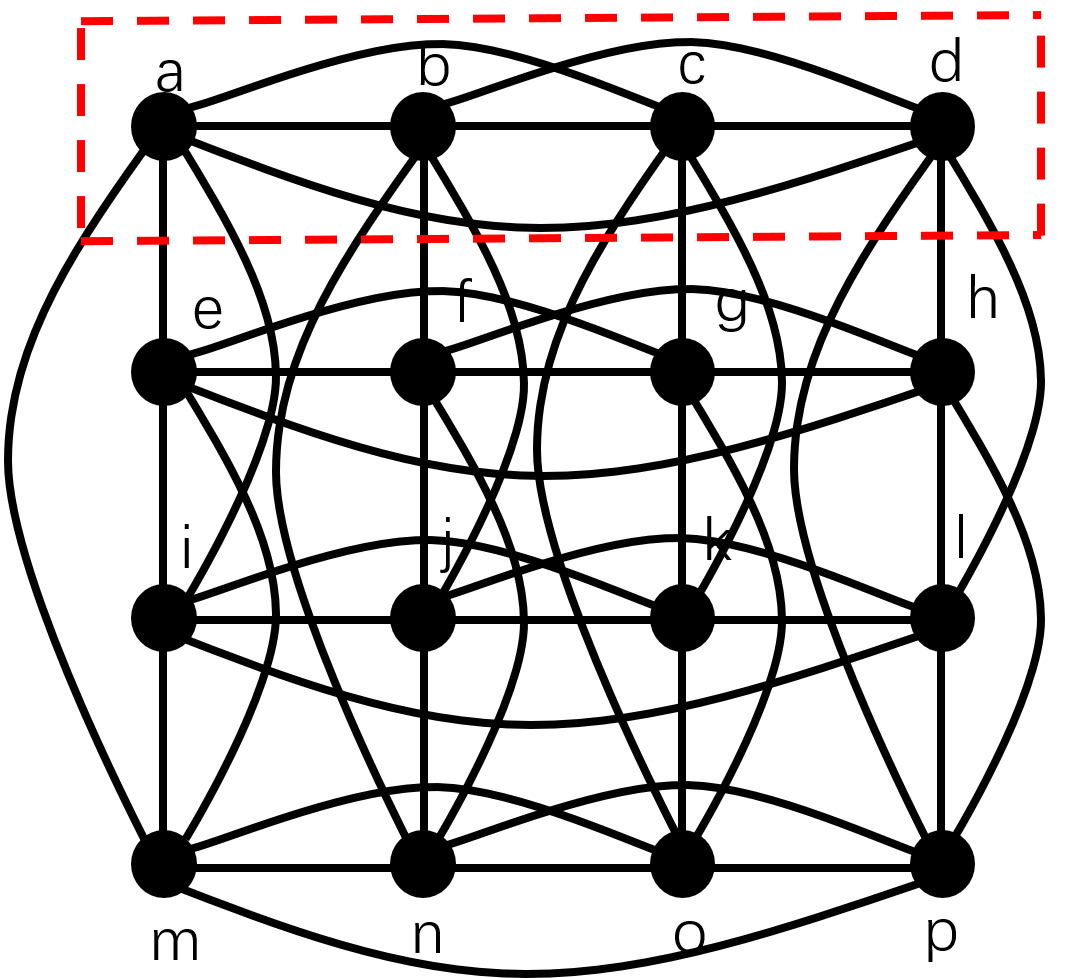}
		\end{minipage}%
	}%
	\subfigure[]{
		\begin{minipage}[t]{0.37\linewidth}
			\centering
			\includegraphics[width=\columnwidth]{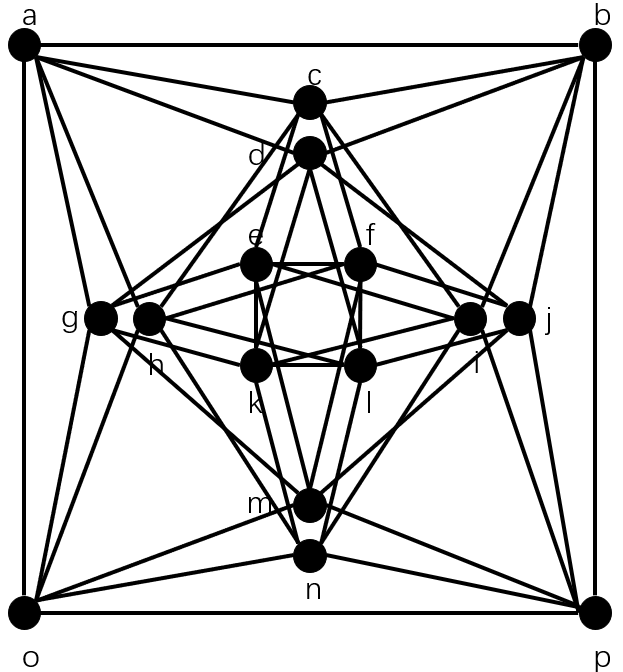}
		\end{minipage}%
	}%
	\centering
	\vspace{-.1in}
	\caption{(a) the 4*4 Rook Graph and (b) the Shrikhande Graph}
	\label{fig:3wl}
	\vspace{-.15in}
\end{figure}

 \begin{proof}
 \textbf{ESC-GNN is not less powerful than 3-WL.} As shown in Theorem~\ref{th:count}, ESC-GNN is able to count 4-cliques. In the pair of graphs called the 4*4 Rook Graph and the Shrikhande Graph (shown in Figure~\ref{fig:3wl}), there exist several 4-cliques in the 4*4 Rook Graph, while there is no 4-clique in the Shrikhande Graph. Therefore ESC-GNN can differentiate the pair of graphs. Considering that 3-WL cannot differentiate them~\citep{arvind2020weisfeiler},  ESC-GNN is not less powerful than 3-WL.
 
 \textbf{ESC-GNN is more powerful than 2-WL.} Using the MPNN~\citep{xu2018powerful} as the backbone network, it can be as powerful as 2-WL in terms of distinguishing non-isomorphic graphs. However, there exist pairs of graphs, e.g., the 4*4 Rook Graph and the Shrikhande Graph that can be distinguished by ESC-GNN but not 2-WL. Therefore, ESC-GNN is strictly more powerful than the 2-WL.

 \end{proof}
 
\section{The Proof of Theorem~\ref{th:count}}
Below we show the counting power of ESC-GNN in terms of subgraph counting. 

\begin{figure*}[btp]
\begin{center}
\centerline{\includegraphics[width=0.8\columnwidth]{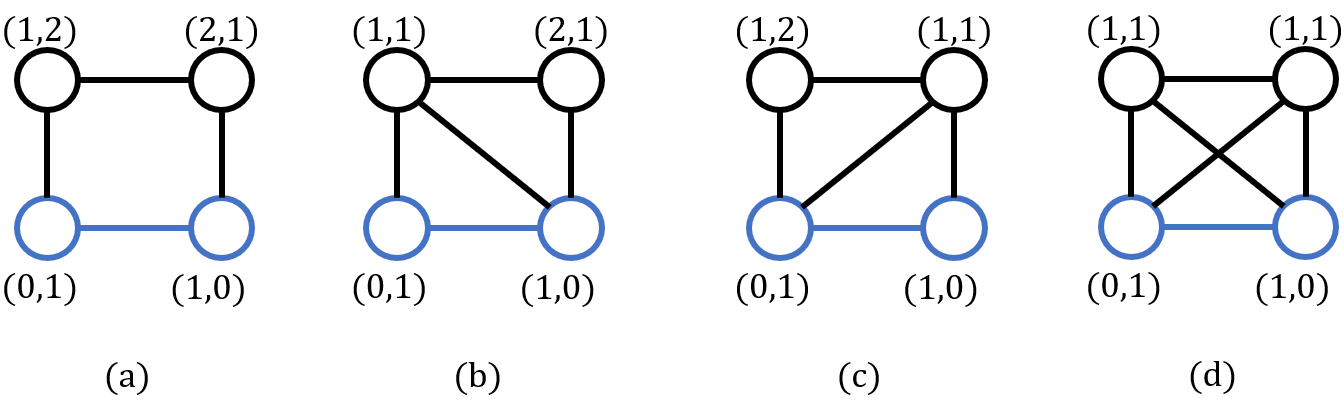}}
\caption{Examples of 4-cycles that pass the rooted edges. In these figures, the rooted 2-tuples are colored blue.}
\label{fig:4cycle}
\end{center}
\end{figure*}

\begin{theorem}
In terms of subgraph counting, ESC-GNN can count (1) up to 4-cycles; (2) up to 4-cliques; (3) stars with arbitrary sizes; (4) up to 3-paths.
\end{theorem}

\begin{figure*}[btp]
\begin{center}
\centerline{\includegraphics[width=0.5\columnwidth]{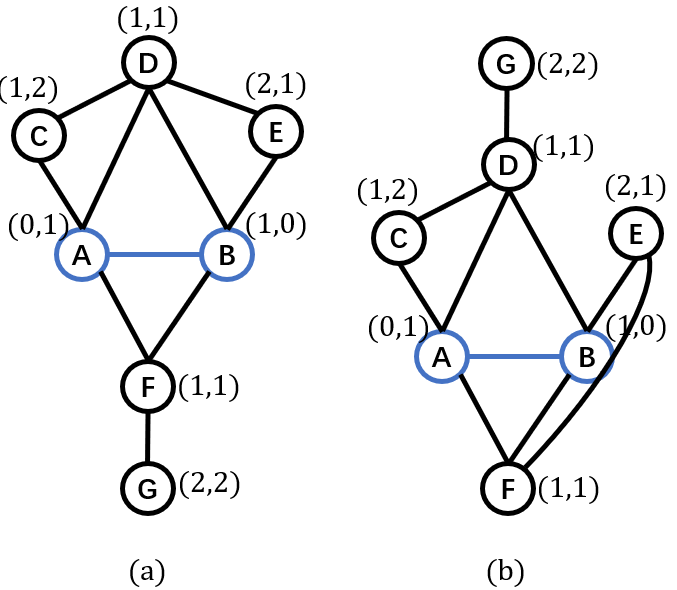}}
\caption{Examples where ESC-GNN cannot subgraph-count 5-cycles. In these figures, the rooted 2-tuples are colored blue.}
\label{fig:5cycle_sub}
\end{center}
\end{figure*}

\begin{proof}
\textbf{Clique counting.} The number of 3-cliques is the number of nodes with the shortest path distance ``1" to both rooted nodes; the number of 4-cliques is the number of edges labeled $(1,1,1,1)$ (definition see the edge-level distance encoding in Section 4 in the main paper, and example see Figure~\ref{fig:4cycle}(d)). Therefore, ESC-GNN can count these types of cliques. In terms of 5-cliques, 4-WL cannot count them according to Theorem~\ref{th:disab}, therefore subgraph GNNs rooted on edges with MPNN as the backbone GNN cannot count 5-cliques according to Proposition 3.5 in the main paper. Then according to Proposition 4.1 in the main paper, ESC-GNN also cannot count 5-cliques.

\textbf{Cycle counting.} the counting of 3-cycles is the same as 3-cliques. In terms of counting 4-cycles, there are basically 4 different situations where 4-cycles exist, examples are shown in Figure~\ref{fig:4cycle}. Note that figures (a),(b), and (c) contain one 4-cycle, respectively, while figure (d) contains two 4-cycles that pass the rooted edge. Therefore the number of 4-cycles is the weighted sum of the number of edges with labels $(1,2,2,1)$, $(1,1,2,1)$, $(1,2,1,1)$, $(1,1,1,1)$.

In terms of 5-cycle subgraph-counting, we provide a counter-example in Figure~\ref{fig:5cycle_sub}. In Figure~\ref{fig:5cycle_sub}(a), there is one 5-cycle $ABEDC$ that passes $AB$, while there is no 5-cycle that passes $AB$ in Figure~\ref{fig:5cycle_sub}(b). Considering that the degree information and the distance information is the same for the pair of graphs, ESC-GNN cannot differentiate them. Therefore, ESC-GNN cannot subgraph-count 5-cycles.

\textbf{Path Counting.} Note that the definition of paths here is different from other substructures. It denotes the number of edges within the path instead of the number of nodes within the path. Here, we slightly extend the use of 2-tuples in ESC-GNN by considering not only 2-tuples with edges but also 2-tuples without edges.

In terms of counting 2-paths or 3-paths between 2 nodes, it is equal to counting 3-cycles or 4-cycles, between edges, respectively. We have proven that ESC-GNN can count these cycles, therefore ESC-GNN can count such edges.

\textbf{Star counting.} We can decompose the graph-level star counting problem to 2-tuples by considering the first node of each 2-tuple as the root of stars. Examples are shown in Figure~\ref{fig:star}. We advocate that the number of stars is easily encoded by the number of nodes whose shortest path distance is 1 to the first rooted node. Denote the number of nodes with the shortest path distance 1 to the first rooted node as $N'$ (including the second node), then the number of $p$-stars is $C_{N'-1}^{p-2}$. A similar proof is provided by~\citep{chen2020can}.

\end{proof}

\begin{figure*}[btp]
\begin{center}
\centerline{\includegraphics[width=0.7\columnwidth]{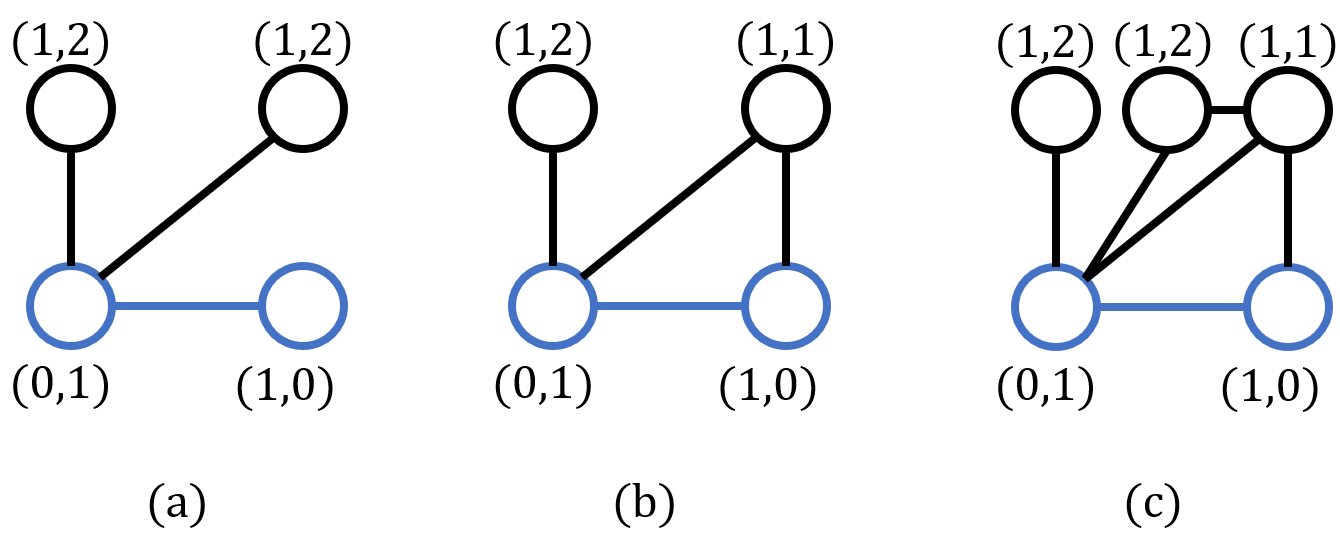}}
\caption{Examples of stars that pass the rooted edges. In these figures, the rooted 2-tuples are colored blue.}
\label{fig:star}
\end{center}
\end{figure*}

\section{The Proof of Theorem~\ref{th:count_ind}}

Below we show the counting power of ESC-GNN in terms of induced subgraph counting. We restate the Theorem~\ref{th:count_ind} as follows.

\begin{theorem}
In terms of induced subgraph counting, ESC-GNN can count (1) up to 4-cycles; (2) up to 4-cliques; (3) up to 4-stars; (4) up to 3-paths.
\end{theorem}

\begin{figure*}[btp]
\begin{center}
\centerline{\includegraphics[width=0.6\columnwidth]{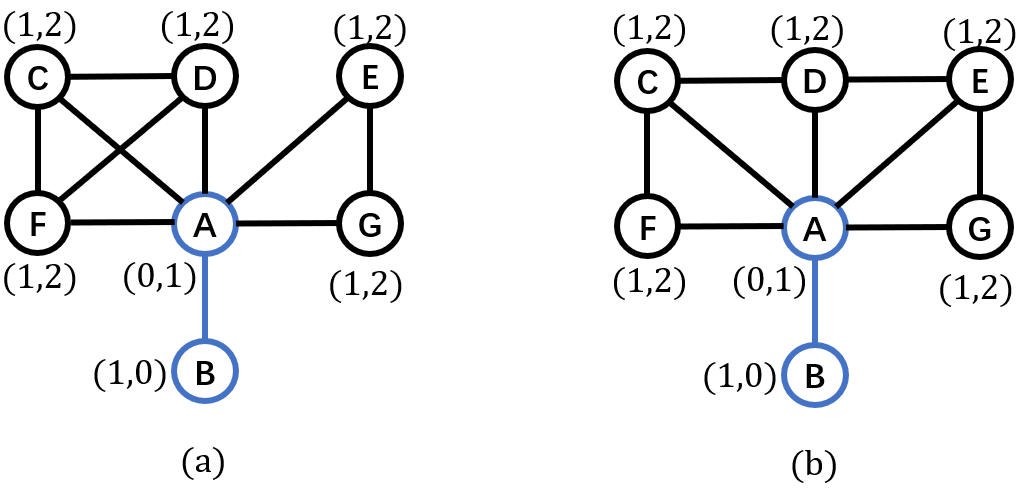}}
\caption{Examples where ESC-GNN cannot induced-subgraph-count 5-stars. In these figures, the rooted 2-tuples are colored blue.}
\label{fig:5star}
\end{center}
\end{figure*}

\begin{proof}
\textbf{Clieque counting.} Since cliques are fully connected substructures, the proof of cliques is the same as the proof of subgraph counting.

\textbf{Cycle counting.} For cycles, the number of 3-cycles is the same as 3-cliques. As for 4-cycles, we only need to consider the situation shown in Figure~\ref{fig:4cycle}(a), where the number of $(1,2,2,1)$ edges reflects the number of 4-cycles. For induced-subgraph-counting 5-cycles, in Figure~\ref{fig:5cycle}(a), there is no 5-cycle that pass $AB$, while in Figure~\ref{fig:5cycle}(b), there is one 5-cycle $ABEDC$ that passes $AB$. However, ESC-GNN cannot differentiate the two graphs since the degree information and the distance information is the same. Therefore, ESC-GNN cannot induced-subgraph-count 5-cycles. It can serve as the same example for not counting 4-paths.

\textbf{Path counting.} The proof of paths is actually the same as Theorem~\ref{th:count}.

\textbf{Star counting.} In terms of 3-stars and 4-stars, we only need to consider the situation shown in Figure~\ref{fig:star}(a), where the number of $(1,2)$ nodes encodes the number of stars, i.e., denote the number of $(1,2)$ nodes as $N'$, the number of $p$-stars ($p \leq 4$) is $C_{N'}^{p-2}$. 

For 5-stars, a pair of examples are shown in Figure~\ref{fig:5star}. These two graphs will be assigned the same structural embedding since the node degree information, and the distance information among these two graphs are the same. Therefore, ESC-GNN cannot differentiate the two graphs. However, Figure~\ref{fig:5star}(a) contains no 5-star that passes the rooted 2-tuple $AB$, while Figure~\ref{fig:5star}(b) contains one 5-star ($BAFDG$) that passes the rooted 2-tuple $AB$.
\end{proof}

\begin{figure*}[btp]
\begin{center}
\centerline{\includegraphics[width=0.55\columnwidth]{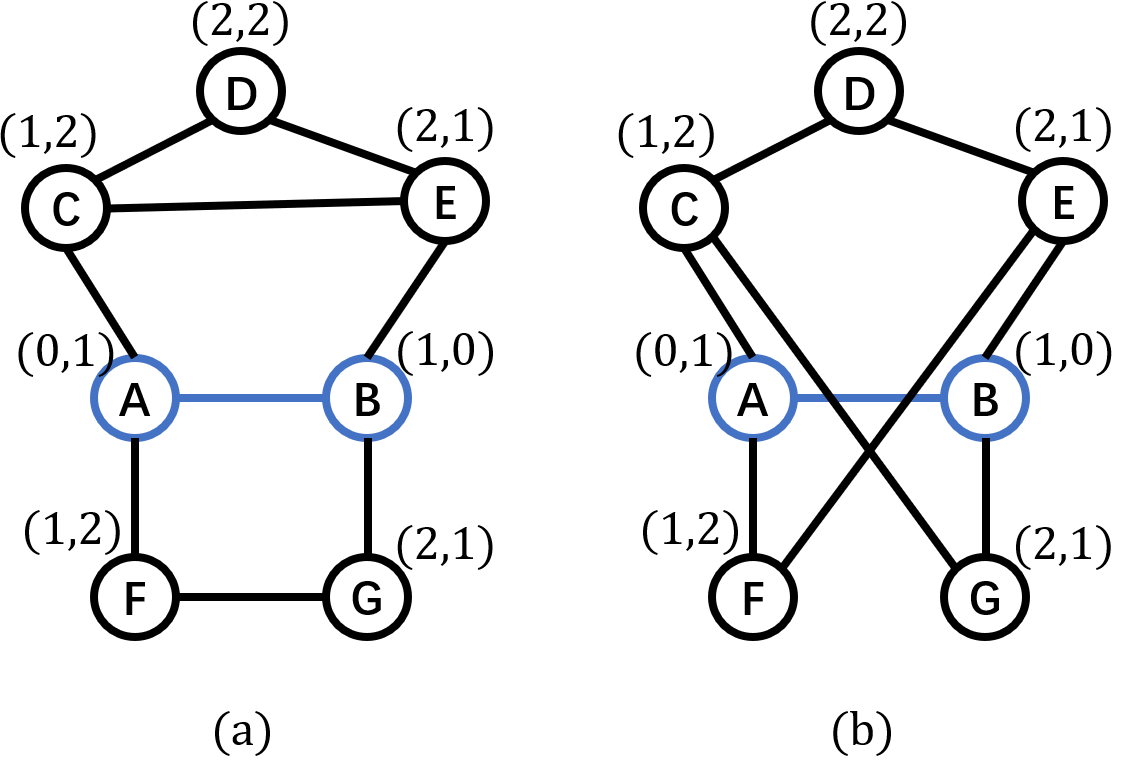}}
\caption{Examples where ESC-GNN cannot induced-subgraph-count 5-cycles. In these figures, the rooted 2-tuples are colored blue.}
\label{fig:5cycle}
\end{center}
\end{figure*}

\section{The Proof of Theorem~\ref{th:reg}}

Here we restate Theorem~\ref{th:reg} as follows: 

\begin{theorem}
Consider all pairs of $r$-regular graphs with $n$ nodes, let $3\leq r < (2log2n)^{1/2}$ and $\epsilon$ be a fixed constant. With the hop parameter $h$ set to $\lfloor(1/2 + \epsilon)\frac{log2n}{log(r-1)}\rfloor$, there exists an ESC-GNN that can distinguish $1-o(n^{-1/2})$ such pairs of graphs.
\end{theorem}

Recall that we encode the distance information as augmented structural features for ESC-GNN. For the input graph $G$, denote the node-level distance encoding on $h$-hop subgraphs on 2-tuple $(u,v)$ as $D_{(u,v),G}^{h,\text{node}}$. Note that we can naturally transfer the encoding on 2-tuples to node by: ($D_{v,G}^{h,\text{node}} = D_{(v,v),G}^{h,\text{node}}$). We then introduce the following lemma.

\begin{lemma}
\label{lemma:reg}
For two graphs $G_1 = (V_1, E_1)$ and $G_2 = (V^2, E^2)$ that are randomly and independently sampled from $r$-regular graphs with n nodes ($3 \leq r < (2log2n)^{1/2}$). Select two nodes $v_1$ and $v_2$ from $G_1$ and $G_2$ respectively. Let $h = \lfloor(1/2 + \epsilon)\frac{log2n}{log(r-1)}\rfloor$ where $\epsilon$ is a fixed constant, $D_{v_1,G_1}^{h,\text{node}} = D_{v_2, G_2}^{h, \text{node}}$ with probability at most $o(n^{-3/2})$.
\end{lemma}

\begin{proof}
The proof follows from~\citep{bollobas1982distinguishing, feng2022powerful}. As $D_{v,G}^{h,\text{node}}$ stores exactly the same information as the node configuration used in~\citep{feng2022powerful}, therefore the proof is exactly the same as the proof of Lemma 1 in~\citep{feng2022powerful}. 

Based on Lemma~\ref{lemma:reg}, we can prove Theorem~\ref{th:reg}. Given node $v_1 \in V_1$, compare $D_{v_1,G_1}^{h,\text{node}}$ with each $D_{v_2,G_2}^{h,\text{node}}$ where $v_2 \in V_2$. The probability that $D_{v_1,G_1}^{h,\text{node}} \neq D_{v_2,G_2}^{h,\text{node}}$ for all possible $v_2 \in V_2$ is $1-o(n^{-3/2})*n = 1 - o(n^{-1/2})$. Therefore, ESC-GNN can distinguish $1-o(n^{-1/2})$ such pairs of graphs.
\end{proof}

\section{Discussion on ESC-GNN's Expressive Power Through Message Passing}
\label{sec:th_mp}
In previous sections, we mainly discuss the expressive power enhanced by the proposed structural encoding. In this section, we shift to the integration of the global message passing layer and elucidate its substantial contribution to the expressive power. This is exemplified below through a toy example: consider graph $G$, composed of a 4-cycle and a 1-path (edge), and graph $H$, comprising a 5-path (a path with 6 nodes). When extracting 1-hop subgraphs for each node, both $G$ and $H$ yield two 1-paths and four "V graphs" (3 nodes and 2 edges). Consequently, the subgraph representation fails to distinguish between these two graphs. However, the introduction of a global message passing layer enables differentiation due to the distinct connections among these subgraphs. This distinction is applicable to both subgraph GNNs and our framework. \citet{rattan2023weisfeiler} provides a more comprehensive analysis, which substantiates that the global message passing can enhance expressiveness regardless of the choice of the hop parameter.

Additionally, We conduct an ablation study on the substructure counting dataset by removing the global message passing layer from ESC-GNN. The resulting new model, denoted as "(-MP)", relies solely on aggregating the structural encoding to derive the final prediction. Given the absence of additional node attributes, the new model can be easily implemented as ESC-GNN with a single message passing layer that aggregates the edge-level structural embedding to obtain the node-level prediction. The result is presented in Table~\ref{tab:count_ab_mp}. As evident from the table, ESC-GNN consistently outperforms the new model across all benchmarks, demonstrating that the global message passing layer significantly boosts the representation power of ESC-GNN.

\begin{table*}
	\centering
	\caption{Ablation study on the proposed structural embedding (norm MAE).} 
    \label{tab:count_ab_mp}
	\scalebox{1.0}{
	\begin{tabular}{lccccccccc}
	\hline\noalign{\smallskip}
	Dataset & Tailed Triangle & Chordal Cycle & 4-Clique & 4-Path & Triangle-Rectangle & 3-cycles & 4-cycles & 5-cycles & 6-cycles \\
	\noalign{\smallskip}\hline\noalign{\smallskip}
	MPNN & 0.3631 & 0.3114 & 0.1645 & 0.1592 & 0.2979 &  0.3515 & 0.2742 & 0.2088 & 0.1555\\
 	\noalign{\smallskip}\hline\noalign{\smallskip}
	ESC-GNN & \cellcolor{yellow}0.0052 & 0.0169 & \cellcolor{yellow}0.0064 & 0.0254 & 0.0178 & \cellcolor{yellow}0.0074 &\cellcolor{yellow}0.0044 & 0.0356& 0.0337\\
    \noalign{\smallskip}\hline\noalign{\smallskip}
     (- MP) & 0.062 & 1.9355 & 0.4124& 0.0271 &0.0563 & 0.1508 & \cellcolor{yellow}0.0050 & 0.0996 &  0.0443 \\
	\hline\noalign{\smallskip}
    \end{tabular}}
\end{table*}

\section{Discussion on Substructure Counting and Real-World Representation Power}
\label{sec:count_real}
In this section, we discuss whether ESC-GNN's representation power only benefits from its substructure-counting ability or not.

\textbf{Importance of substructures.} We incorporate an ablation study to show the importance of substructures in real-world tasks. The substructures, similar to Table~\ref{tab:count}, include 3-6 cycles, tailed triangles, chordal cycles, 4-cliques, 4-paths, and triangle-rectangles. We first report the number of these substructures at graph-level and node-level. Specifically, we report the average number of these substructures per graph (denoted as "graph") and the average number of these substructures that pass each node (denoted as "node"). The statistics are shown in Table~\ref{tab:cycle_ogb} and Table~\ref{tab:cycle_zinc}.

\begin{table*}
	\centering
	\caption{The average number of substructures per graph/node in OGBG-HIV}
	\label{tab:cycle_ogb} 
	\scalebox{1.0}{
		\begin{tabular}{lccccccccc}
			\hline\noalign{\smallskip}
			OGBG-HIV & 3-cycle & 4-cycle & 5-cycle & 6-cycle & tailed triangle & chordal cycle & 4-clique & 4-path & triangle-rectangle\\
			\noalign{\smallskip}\hline\noalign{\smallskip}
			graph & 2.69e-2 & 3.69e-2 & 7.00e-1 & 2.28 & 1.02e-1 & 4.84e-3 & 2.43e-5 & 68.38 & 9.36e-3 \\
			node & 3.16e-3 & 5.79e-3 & 1.37e-1 & 5.37e-1 & 4.01e-3 & 3.79e-4 & 7.63e-6 & 5.40 & 3.67e-4\\
			\noalign{\smallskip}
			\hline
			\noalign{\smallskip}
	\end{tabular}}
\end{table*}

\begin{table*}
	\centering
	\caption{The average number of substructures per graph/node in ZINC}
	\label{tab:cycle_zinc} 
	\scalebox{1.0}{
		\begin{tabular}{lccccccccc}
			\hline\noalign{\smallskip}
			ZINC & 3-cycle & 4-cycle & 5-cycle & 6-cycle & tailed triangle & chordal cycle & 4-clique & 4-path & triangle-rectangle\\
			\noalign{\smallskip}\hline\noalign{\smallskip}
			graph & 6.41e-2 & 1.60e-2 & 8.52e-1 & 1.80 & 8.79e-2 & 0 & 0 & 55.15 & 0 \\
			node & 8.30e-3 & 2.76e-3 & 1.84e-1 & 4.67e-1 & 3.80e-3 & 0 & 0 & 4.76 & 0\\
			\noalign{\smallskip}
			\hline
			\noalign{\smallskip}
	\end{tabular}}
\end{table*}

Note that 6-cycles are commonly observed within these two graphs. On average, a graph contains approximately two 6-cycles, and there is an average incidence of 0.5 6-cycles that pass a node. This observation underscores the significance of 6-cycle, whose presence may potentially be attributed to the existence of benzene rings.

We then evaluate the significance of these substructures. In particular, we use a base GIN framework as the backbone network, and adopt the count of substructures that pass each node as the augmented node features. In Table~\ref{tab:ab_struc}, “all” denotes adding all these substructures as the augmented node feature. Due to the time limit, we only run the experiment twice and report the mean result.

\begin{table*}
	\centering
	\caption{Analysis on the importance of substructures on real-world benchmarks}
	\label{tab:ab_struc} 
	\scalebox{0.95}{
		\begin{tabular}{lccccccccccc}
			\hline\noalign{\smallskip}
			Dataset & Baseline & 3-cycle & 4-cycle & 5-cycle & 6-cycle & tailed triangle & chordal cycle & 4-clique & 4-path & triangle-rectangle & all\\
			\noalign{\smallskip}\hline\noalign{\smallskip}
			ZINC & 0.149 & 0.145 & 0.150 & 0.143 & 0.115 & 0.146 & 0.152 & 0.149 & 0.167 & 0.149 & 0.100 \\
			OGBG-HIV & 73.89 & 73.48 & 74.67 & 74.47 & 73.60 & 73.69 & 71.98 & 73.56 & 68.40 & 53.18 & 71.47\\
			\noalign{\smallskip}
			\hline
			\noalign{\smallskip}
	\end{tabular}}
\end{table*}

\textbf{Analysis of the importance of substructures on real-world benchmarks.} We observe that in most cases, the number of substructures does not boost the performance either on ZINC or on OGBG-HIV. The only exception is that the number of 6-cycles boosts the performance on ZINC (maybe due to the information of benzene rings), and the number of all these substructures boosts the performance on ZINC. We present potential illustrations for the observation:

Firstly, in these molecule graphs, it is imperative to incorporate not only structural but also semantic information to enhance the model. For instance, it is well-established that molecular fingerprints can significantly elevate performance on the OGBG-HIV dataset (See \url{https://ogb.stanford.edu/docs/leader_graphprop/#ogbg-molhiv}). In particular, the fingerprints contain the number of specific molecule structures, encompassing both the graph structures of molecules and the semantic information of atom types and bond types.

Secondly, the general distance information contains more structural information than the number of specific substructures, thus contributing more to real-world performance. For instance, we have proven that the number of 4-cycles can be computed using the distance information. Conversely, deducing distance information solely from the count of 4-cycles is impracticable, given the multiplicity of configurations that can form a 4-cycle.

\section{Experimental Details}
\label{sec:exp_detail}
\textbf{Stastics of Datasets.} The statistics of all used datasets are available in Table~\ref{tab:stat}.

\begin{table*}
	\centering
	\caption{Statistics of the used datasets.}
	\label{tab:stat} 
	\scalebox{1.0}{
		\begin{tabular}{lccccc}
			\hline\noalign{\smallskip}
			Dataset & Graphs & Avg Nodes & Avg Edges & Task Type & Metric\\
			\noalign{\smallskip}\hline\noalign{\smallskip}
			
			ZINC-12k & 12000 & 23.2 & 24.9 & Graph Regression & MAE\\
            ZINC-250k & 249456 & 23.1 & 24.9 & Graph Regression & MAE\\
			OGBG-HIV & 41127 & 25.5 & 27.5 & Graph Classification & AUC-ROC\\
            
            QM9 & 129433 & 18.0 & 18.6 & Graph Regression & MAE\\
            Synthetic & 5000 & 18.8 & 31.3 & Node Regression & MAE\\
			\noalign{\smallskip}
			\hline
			\noalign{\smallskip}
	\end{tabular}}
\end{table*}

\textbf{Experimental details.} The baselines and data splittings of our experiments follow from existing works~\citep{zhang2021nested, huang2023boosting} and the standard data split setting. For ESC-GNN, we adopt GIN~\citep{xu2018powerful} as the backbone GNN (the only exception is that in Table~\ref{tab:mol}, we use a plain graph transformer~\cite{rampavsek2022recipe} as the backbone model for ZINC). In the structural embedding, we use both the shortest path distance and the resistance distance~\citep{lu2011link} as the distance feature. For the hop parameter $h$ we search between 1 to 4, and report the best results. Following existing works~\citep{huang2023boosting}, we use Adam optimizer as the optimizer, and use plateau scheduler with patience 10 and decay factor 0.9. On most datasets, the learning rate is set to 0.001, and the hidden embedding dimension is set to 300. The training epoch is set to 2000 for counting substructures, 400 for QM9, 1000 for ZINC, and 150 for the OGB dataset. Most of the experiments are implemented with two Intel Core i9-7960X processors and 2 NVIDIA 3090 graphics cards. Others (e.g., experiments on the TU dataset) are implemented with two Intel Xeon Gold 5218 processors and 10 NVIDIA 2080TI graphics cards. 

\section{Evaluation on the Space Cost}

With regards to the preprocessing time, our approach's preprocessing time is comparable to that of I$^2$-GNN as shown in Table 3 in the main paper since we use their code to extract subgraph information. Although we require a small amount of extra time to extract and preprocess the structural embeddings, we believe that this is a reasonable trade-off since the actual model running time is reduced by orders. When compared to the total running time of subgraph GNNs, the preprocessing time is negligible. Therefore, our ESC-GNN's total running time is less than 1\% of that of I$^2$-GNN on ogbg-hiv.

As for the storage of structural embeddings, we only need to store a vector of integer indices for each structural embedding, as illustrated in Figure 1 in the main paper, without really storing any dense high-dimensional vectors. Therefore, the storage is still manageable. Specifically, when feeding the indices to the backbone model, we use a learnable matrix (which is a model parameter) to transform the integer index vector into dense embeddings. For example, to extract the degree information of the first subgraph in Figure 1 in the main paper, we use a learnable matrix $W \in \mathbb{R}^{4\times h}$ and compute the degree information with $W * [0;0;4;0]$, where $*$ denotes sparse matrix multiplication (which is very fast for sparse matrices), and the sparse vector $[0;0;4;0]$ is what we need to store. This approach requires relatively small storage space.

We have also included the space cost of our approach in Table~\ref{tab:space}. In these datasets, ESC-GNN requires much less storage space than I$^2$-GNN while slightly more space than NGNN. These results demonstrate that our approach for storing structural embeddings offers excellent storage performance.

\begin{table*}
	\centering
	\caption{Evaluation on the space cost.}
	\label{tab:space} 
	\scalebox{1.0}{
		\begin{tabular}{lccccc}
			\hline\noalign{\smallskip}
			Dataset & OGBG-HIV &ZINC & QM9\\
			\noalign{\smallskip}\hline\noalign{\smallskip}
			Original & 159MB & 16.2MB & 281MB\\
   			NGNN & 2.32GB & 218MB & 2.90GB\\
			I$^2$-GNN & 5.95GB & 627MB & 8.25GB \\
			ESC-GNN & 2.57GB & 427MB & 4.46GB\\

			\noalign{\smallskip}
			\hline
			\noalign{\smallskip}
	\end{tabular}}
\end{table*}

\section{Limitations and the Assets We Used}

\textbf{Limitations of the paper.} First, we have shown that the representation power of our model is bounded by 4-WLs and subgraph GNNs rooted on 2-tuples in terms of distinguishing non-isomorphic graphs and counting substructures.

Second, the proposed model may not reach a satisfying performance on benchmarks where the encoded substructures are of no use. Also, the proposed model may not suit high-order graphs where the neighbors of the nodes and edges are defined differently from the simple graphs. 

\textbf{The assets we used.} Our model is experimented on benchmarks from~\citep{dwivedi2020benchmarking, hu2020open, zhao2022stars, morris2020tudataset, abboud2021surprising, balcilar2021breaking, murphy2019relational, ramakrishnan2014quantum, wu2018moleculenet} under the MIT license.


\end{document}


\maketitle

In the supplementary material, we provide:
\begin{enumerate}
\item the introduction of the Weisfeiler-Leman algorithm (WL) and Message Passing Neural Network (MPNN);
\item the proof of Theorem 3.2 in the main paper (the upper bound of $m$-WL in terms of counting substructures);
\item the proof of Theorem 3.4 in the main paper (the lower bound of subgraph GNN in terms of counting connected substructures)
\item the proof of Theorem 4.4 in the main paper (ESC-GNN's ability to differentiate non-isomorphic graphs);
\item the proof of Theorem 4.2 and Theorem 4.3 in the main paper (ESC-GNN's ability for subgraph counting and induced subgraph counting); 
\item the proof of Theorem 4.5 (ESC-GNN's ability to differentiate regular graphs);
\item additional discussion on the expressive power of ESC-GNN;
\item experimental details;
\item additional experiments on real-world benchmarks and ablation study;
\item limitations and the assets we used.
\end{enumerate} 

\section{The Weisfeiler-Leman algorithm and Message Passing Neural Network}
\textbf{$1$-WL.} We first describe the classic Weisfeiler-Leman algorithm (1-WL). For a given graph $G$, 1-WL aims to compute the coloring for each node in $V$. It computes the node coloring for each node by aggregating the color information from its neighborhood iteratively. In the 0-th iteration, the color for node $v \in V$ is $c_v^0$, denoting its initial isomorphic type\footnote{The term "isomorphic type" is based on previous work~\citep{morris2019weisfeiler}, which gives each $k$-tuple of nodes an initial feature such that two $k$-tuples receive the same initial feature if and only if their induced subgraphs (indexed by node order in the $k$-tuple) are isomorphic.}. For labeled graphs, the isomorphic type of a node is simply its node feature. For unlabeled graphs, we give the same 1 to all nodes. In the $t$-th iteration, the coloring for $v$ is computed as $$c_v^t = \text{HASH}(c_v^{t-1}, \{\!\!\{c_u^{t-1}:u \in N(v)\}\!\!\})$$,
where HASH is a bijective hashing function that maps the input to a specific color. The process ends when the colors for all nodes between two iterations are unchanged. If two graphs have different coloring histograms in the end (e.g., different numbers of nodes with the same color), then $1$-WL detects them as non-isomorphic.

\textbf{$k$-WL.} For each $k\geq 2, k \in \mathbb{N}$, the $k$-dimensional WL algorithm ($k$-WL) colors $k$-tuples instead of nodes. In the 0-th iteration, the isomorphic type of a $k$-tuple $\vec{v}$ is given by the hashing of 1) the tuple of colors associated with the nodes of the $\vec{v}$, and 2) the adjacency matrix of the subgraph induced by $\vec{v}$ ordered by the node ordering within $\vec{v}$. 
In the $t$-th iteration, its coloring is updated by: 
$$c_{\vec{v}}^t = \text{HASH}(c_{\vec{v}}^{t-1}, c_{\vec{v}, (1)}^{\textcolor{black}{t}},...,c_{\vec{v}, (k)}^{\textcolor{black}{t}})$$,
where $$c_{\vec{v}, (i)}^t = \{\!\!\{c_{\vec{u}}^{t-1}|\vec{u} \in N_i(\vec{v})\}\!\!\}, i \in [k]$$. 
Here, $N_i(\vec{v}) = \{(v_1,...,v_{i-1}, w, v_{i+1},...,v_k) | w \in V\}$ is the $i$-th neighborhood of $\vec{v}$. Intuitively, $N_i(\vec{v})$ is obtained by replacing the $i$-th component of $\vec{v}$ by each node from $V$. Besides the updating function, other procedures of $k$-WL are analogous to 1-WL. In terms of distinguishing graphs, 2-WL is as powerful as 1-WL, and for $k\geq 2$, $(k+1)$-WL is strictly more powerful than $k$-WL.

\textbf{MPNNs.} MPNNs are a class of GNNs that learns node representations by iteratively 
aggregating messages from neighboring nodes. Let $h_v^t$ be the node representation for $v \in V$ in the $t$-th iteration. It is usually initialized with the node's intrinsic attributes. In the $(t+1)$-th iteration, it is updated by $$h_v^{t+1} = W_1^t(h_v^t, \sum_{u \in N(v)}W_2^t(h_u^t, h_v^t, e_{uv}))$$,
where $W_1^t$ and $W_2^t$ are two learnable functions. MPNN's power in terms of distinguishing non-isomorphic graphs is upper bounded by 2-WL~\citep{xu2018powerful}.

\section{The Proof of Theorem 3.2}

We begin by introducing the \textit{graph isomorphism}. For a pair of graphs $G_1 = (V_1, E_1)$ and $G_2 = (V_2, E_2)$, if there exists a bijective mapping $f: V_1 \rightarrow V_2$, so that for any edge $(u_1, v_1) \in E_1$, it satisfies that $(f(u_1), f(v_1)) = (u_2, v_2) \in E_2$, then $G_1$ is isomorphic to $G_2$, otherwise they are not isomorphic. Up to now, there is no polynomial algorithm for solving the graph isomorphism problem. One popular method is to use the $k$-order Weisfeiler-Leman~\citep{weisfeiler1968reduction} algorithm ($k$-WL). It is known that  1-WL is as powerful as 2-WL, and for $k \geq 2$, $(k+1)$-WL is more powerful than $k$-WL.

We then prove that $k$-WL can't count all connected substructures with $(k+1)$ nodes (specifically, $(k+1)$-cliques). We restate the result as follows: 


\begin{theorem}
\label{th:disab}
For any $k\geq 2$, there exists a pair of graphs $G$ and $H$, such that $G$ contains a $(k+1)$-clique as its subgraph while $H$ does not, and that $k$-WL can't distinguish $G$ from $H$.
\end{theorem}

\begin{proof}
	The counter-example is inspired by the well-known Cai-F\"urer-Immerman (CFI) graphs \citep{cai1992optimal}. We define a sequence of graphs $G_k^{(\ell)}, \ell = 0, 1, \ldots, k+1$ as following,
	\begin{equation}
		\begin{split}
			V_{G_k^{(\ell)}} =& \Big\{u_{a, \vec{v}}\Big| a\in[k+1], \vec{v}\in\{0,1\}^k \text{ and }\\
			&\quad\begin{array}{ll}
				\vec{v} \text{ contains an even number of } 1 \text{'s}, & \text{if }a=1,2,\ldots, k-\ell+1, \\
				\vec{v} \text{ contains an odd number of } 1 \text{'s},  & \text{if }a=k-\ell+2,\ldots, k+1.
			\end{array}\Big\}
		\end{split}
	\end{equation}
	Two nodes $u_{a,\vec{v}}$ and $u_{a',\vec{v}'}$ of $G_k^{(\ell)}$ are connected iff there exists $m\in [k]$ such that $a' \mod (k+1) = (a+m) \mod (k+1)$ and $v_m = v'_{k-m+1}$. We have the following lemma.

	\begin{lemma}
		\emph{(a)} For each $\ell = 0, 1, \ldots, k+1$, $G_k^{(\ell)}$ is an undirected graph with $(k+1)2^{k-1}$ nodes;

		\emph{(b)} The set of graphs $G_k^{(\ell)}$ with an odd $\ell$ are mutually isomorphic; similarly, the set of graphs $G_k^{(\ell)}$ with an even $\ell$ are mutually isomorphic.
	\end{lemma}

	It's easy to verify (a). To prove (b), it suffices to prove $G_k^{(\ell)}$ is isomorphic to $G_k^{(\ell+2)}$ for all $\ell = 0,1,\ldots, k-1$. We apply a \emph{renaming} to the nodes of $G_k^{(\ell)}$: we flip the $1^\mathrm{st}$ bit of $\vec{v}$ in every node named $u_{k-\ell, \vec{v}}$, and flip the $k^\mathrm{th}$ bit of $\vec{v}$ in every node named $u_{k-\ell+1, \vec{v}}$. Since this is a mere renaming of nodes, the resulting graph is isomorphic to $G_k^{(\ell)}$. However, it's also easy to see that the resulting graph follows the construction of $G_k^{(\ell+2)}$. Therefore, we assert that $G_k^{(\ell)}$ must be isomorphic to $G_k^{(\ell+2)}$.

	Now, let's ask $G=G_k^{(0)}$ and $H=G_k^{(1)}$. Obviously there is a $(k+1)$-clique in $G$: nodes $u_{j,0^k}, j=1,2, \ldots, k+1$ are mutually adjacent by definition of $G_k^{(0)}$. On the contrary, we have
	\begin{lemma}
		There's no $(k+1)$-clique in $H$.
	\end{lemma}

	The proof is given below. Assume there is a $(k+1)$-clique in $H$. Since there's no edge between nodes $u_{a, \vec{v}}$ with an identical $a$, the $(k+1)$-clique must contain exactly one node from every node set $\{u_{a,\vec{v}}\}$ for each fixed $a\in[k+1]$. We further assume that the $(k+1)$ nodes are $u_{a, b_{a1}b_{a2}\ldots b_{ak}}, a=1,2,\ldots, k+1$. Using the condition for adjacency, we have
	\begin{align}
		       & b_{2k} = b_{11},                                                               \\
		       & b_{3k} = b_{21}, b_{3(k-1)} = b_{12},                                          \\
		       & b_{4k} = b_{31}, b_{4(k-1)} = b_{22}, b_{4(k-2)} = b_{13},                     \\
		\notag & \cdots\cdots\cdots\cdots                                                       \\
		       & b_{(k+1)k} = b_{k1}, b_{(k+1)(k-1)} = b_{(k-1)2}, \ldots, b_{(k+1)1} = b_{1k}.
	\end{align}
	Applying the above identities to the summation
	\begin{align}
		\label{contradiction}
		\sum_{a=1}^{k+1}\sum_{j=1}^k b_{aj} = 2\sum_{j=1}^k\left(b_{1j}+b_{2j}+\cdots+b_{(k-j+1)j}\right),
	\end{align}
	we see that it should be even. However, by definition of $G_k^{(1)}$, there are an even number of $1$'s in $b_{a1}b_{a2}\ldots b_{ak}$ when $a\in [k]$, and an odd number of $1$'s when $a=k+1$. Therefore, the sum in \eqref{contradiction} should be odd. This leads to a contradiction.

	Finally, to prove the $k$-WL equivalence of $G$ and $H$, we have
	\begin{lemma}
		$k$-WL can't distinguish $G$ and $H$.
	\end{lemma}

	By virtue of the equivalence between $k$-WL and pebble games~\citep{grohe_otto_2015}, it suffices to prove that Player II will win the $\mathcal{C}_{k}$ bijective pebble game on $G$ and $H$. We state the winning strategy for Player II as following. Since $G$ and $H$ are isomorphic with nodes $\{u_{k+1,*}\}$ deleted, Player II can always choose an isomorphism $f:G-\{u_{k+1,*}\}\rightarrow H-\{u_{k+1,*}\}$ to survive if Player I never places a pebble on nodes $u_{k+1,*}$. Furthermore, since $k$ pebbles can occupy nodes with at most $k$ different values of $a$ (in $u_{a,\vec{v}}$), there's always a set of pebbleless nodes $\{u_{a_0,\vec{v}}\}$ with some $a_0\in[k+1]$. Therefore, Player II only needs to do proper renaming on $H$ between $u_{k+1,*}$ and $u_{a_0,*}$ as stated above. This makes every $\vec{v}$ in $u_{a_0,\vec{v}}$ have an odd number of $1$'s. Player II then chooses the isomorphism on $G-\{u_{a_0,*}\}$ and $H^{\mathrm{renamed}}-\{u_{a_0,*}\}$. This way, Player II never loses since there are not enough pebbles for Player I to make use of the oddity at the currently pebbleless set of nodes.

\end{proof}

\begin{remark}
    Notice that if we take $k=2$, then $G$ is two 3-cycles while $H$ is a 6-cycle, which 2-WL cannot differentiate; if we take $k=3$, then $G$ is the 4*4 Rook's graph while $H$ is the Shrikhande graph, which 3-WL cannot differentiate. In these special cases, the above construction complies with our well-known examples.
\end{remark}

\section{The Proof of Theorem 3.4}

We restate the theorem as follows:

\begin{theorem}
\label{th:sub}
For any connected substructure with no more than $k+m$ ($m\geq 2, k > 0$) nodes, there exists a subgraph GNN rooted at $k$-tuples with backbone GNN as powerful as $m$-WL that can count it. 
\end{theorem}

\begin{proof}
Based on Remark 3.3 in the main paper, there exists a type of connected $k$-tuple that satisfies the decomposition of the connected substructure. Then the substructure can be separated into 2 subgraphs: the nodes that belong to the $k$-tuple (we call them rooted nodes) and the nodes that do not belong to the $k$-tuple (we call them non-rooted nodes). Formally, for the given two substructures $G_1 = (V_1, E_1)$ and $G_2 = (V_2, E_2)$, we define the subgraph that formed by the rooted nodes of $G_1$ ($G_2$, resp.) as $G_{1,r} = (V_{1, r}, E_{1,r})$ ($G_{2,r} = (V_{2,r}, E_{2,r})$, resp.). Similarly, define the subgraph that formed by the non-rooted nodes of $G_1$ ($G_2$, resp.) as $G_{1,n} = (V_{1, n}, E_{1,n})$ ($G_{2,n} = (V_{2,n}, E_{2,n})$, resp.). We can also define the subgraph formed between the rooted nodes and non-rooted nodes of $G_1$ ($G_2$, resp.) as $G_{1,c} = (V_{1, c}, E_{1,c})$ ($G_{2,c} = (V_{2,c}, E_{2,c})$, resp.). Then it is easy to find that for $G_1$, $V_1 = (V_{1,r} \cup V_{1, n})$, $V_{1,c} \subseteq V_1$, $E_1 = E_{1,r} \cup E_{1, n} \cup E_{1, c}$. There is no intersection between $V_{1,r}$ and $V_{1, n}$, and there is no intersection between $E_{1,r}$, $E_{1,c}$, and $E_{1, n}$. The same holds for $G_2$.

If $G_1$ and $G_2$ are non-isomorphic, then there can be three potential situations: (1) $G_{1,r}$ and $G_{2,r}$ are non-isomorphic; (2) $G_{1,n}$ and $G_{2,n}$ are non-isomorphic; (3) $G_{1,c}$ and $G_{2,c}$ are non-isomorphic. In the following section, we will prove that in all these situations, the subgraph GNN can distinguish between $G_1$ and $G_2$.

\textbf{$G_{1,r}$ and $G_{2,r}$ are non-isomorphic.}  It denotes that for $G_1$ and $G_2$, the selected $k$-tuples are different. Then of course the subgraph GNN can differentiate $G_1$ and $G_2$.

\textbf{$G_{1,n}$ and $G_{2,n}$ are non-isomorphic.} Recall that we use a backbone GNN as powerful as $m$-WL to encode the information within the subgraph, and $V_{1,n}$ and $V_{2,n}$ contain nodes no more than $m$ nodes. Therefore if $G_{1,n}$ and $G_{2,n}$ are non-isomorphic, then they will have different isomorphic types, and thus can be distinguished by the backbone GNN.

\textbf{$G_{1,c}$ and $G_{2,c}$ are non-isomorphic.} We define the label of all nodes in $G_1$ (the same holds for $G_2$) as its distance to the nodes in the $k$-tuple. Formally, let $V_{1,r} = \{v_{1,r,1},...,v_{1,r,k}\}$, then $\forall u \in V_1$, its label $f_1(u) = (d(u, v_{1,r,1}),...,d({u,v_{1,r,k}}), I(u))$, where $d(.)$ denotes the shortest path distance between two nodes, and $I(u)$ denotes the label that reflects the isomorphic type of $u$ encoded by the subgraph GNN within $G_{1,n}$.

Since the substructures are connected, there exists at least a node $u_1 \in V_{1,n}$, whose label contains at least an index with the value ``1". For $f(u_1)$, the indices with value ``1" denotes that there exist edges between $u_1$ and the corresponding nodes in $V_{1,r}$. While the indices with value larger than 1 denote that there is no edge between $u_1$ and the corresponding nodes in $V_{1,r}$. The same holds for $u_2$ and $V_{2,r}$. Therefore, if the subgraph formed by $u_1$ and $V_{1,r}$ and the subgraph formed by $u_2$ and $V_{2,r}$ are non-isomorphic, then $f_1(u_1)$ and $f_2(u_2)$ are different, and the subgraph GNN can differentiate the two substructures. Also, if the $I(u_1)$ and $I_(u_2)$ are different, then the subgraph GNN can also distinguish $G_1$ and $G_2$. We can then consider the next nodes, and continue the process inductively.

Therefore, if for all nodes in $V_{1,n}$, we can find a unique node in $V_{2,n}$ that has the same label as it. Then $G_{1,c}$ and $G_{2,c}$ are isomorphic. Reversely, if $G_{1,c}$ and $G_{2,c}$ are non-isomorphic, then there exists at least a node in $V_{1,n}$, that we cannot find a unique node in $V_{2,n}$ that has the same label as it. Therefore the subgraph GNN can differentiate $G_1$ and $G_2$.

Then we can assign each substructure a unique color according to its isomorphic type, and use the color histogram of the graph as the output function. If two graphs have different numbers of certain substructures, then the color histogram will be different.

Based on the above results, the subgraph GNN can count the given type of connected substructure.

\end{proof}

\section{The Proof of Theorem 4.4}
We state the result as follows.

\begin{theorem}
\label{th:iso}
ESC-GNN is strictly more powerful than 2-WL, while not less expressive than 3-WL.
\end{theorem}

\begin{figure}[btp]
	\centering
	\subfigure[]{
		\begin{minipage}[t]{0.43\linewidth}
			\centering
			\includegraphics[width=\columnwidth]{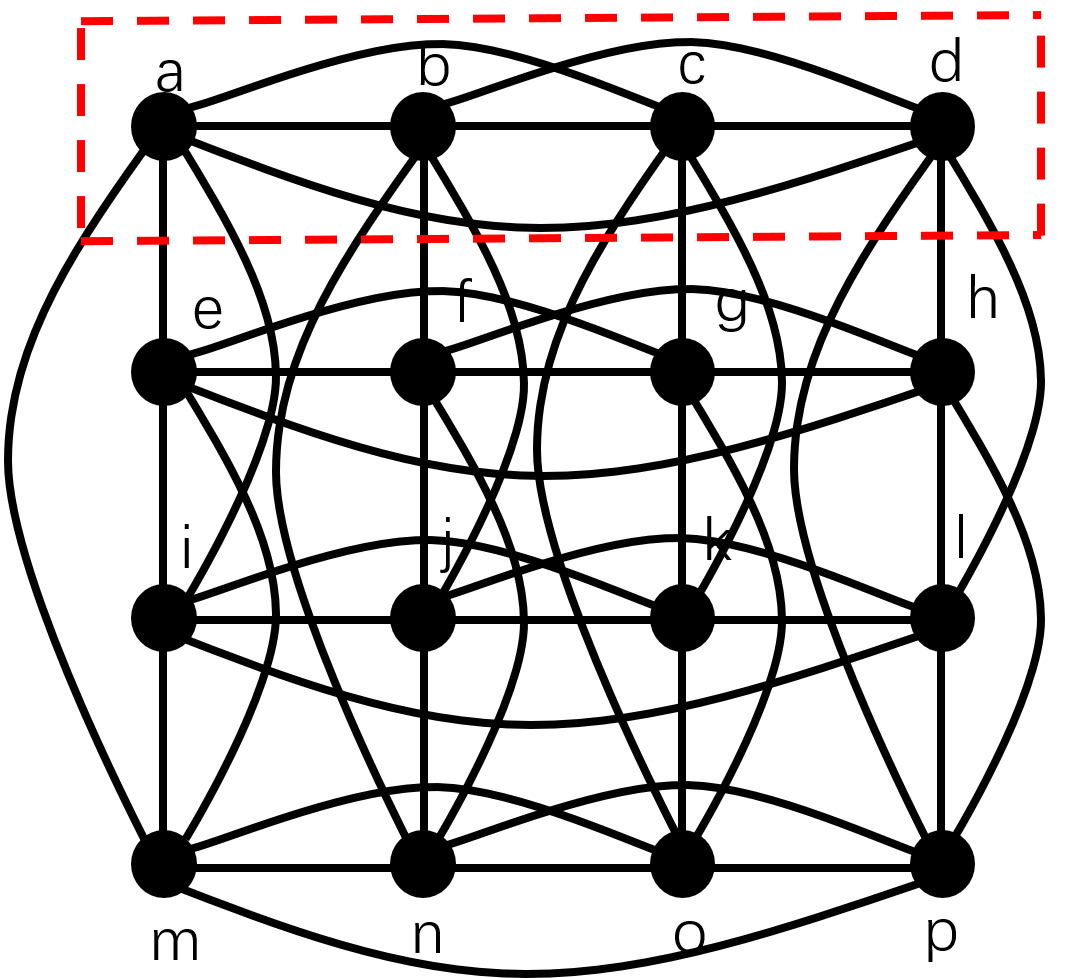}
		\end{minipage}%
	}%
	\subfigure[]{
		\begin{minipage}[t]{0.37\linewidth}
			\centering
			\includegraphics[width=\columnwidth]{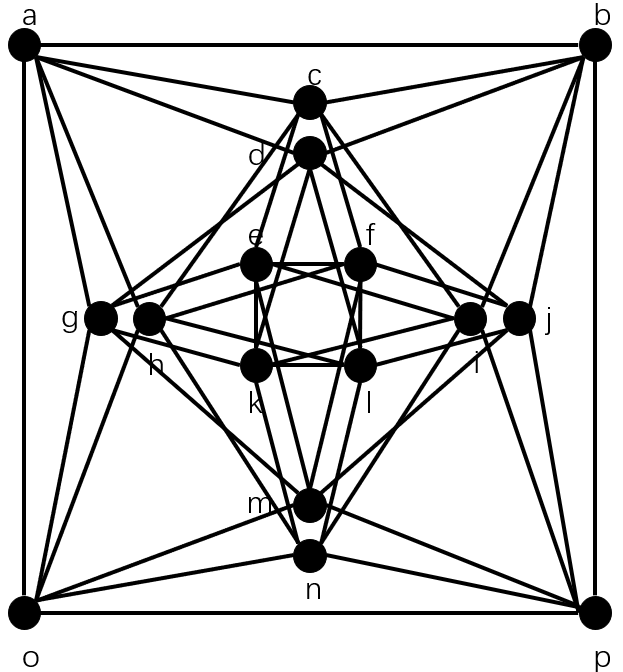}
		\end{minipage}%
	}%
	\centering
	\vspace{-.1in}
	\caption{(a) the 4*4 Rook Graph and (b) the Shrikhande Graph}
	\label{fig:3wl}
	\vspace{-.15in}
\end{figure}

 \begin{proof}
 \textbf{ESC-GNN is not less powerful than 3-WL.} As shown in Theorem~\ref{th:count}, ESC-GNN is able to count 4-cliques. In the pair of graphs called the 4*4 Rook Graph and the Shrikhande Graph (shown in Figure~\ref{fig:3wl}), there exist several 4-cliques in the 4*4 Rook Graph, while there is no 4-clique in the Shrikhande Graph. Therefore ESC-GNN can differentiate the pair of graphs. Considering that 3-WL cannot differentiate them~\citep{arvind2020weisfeiler},  ESC-GNN is not less powerful than 3-WL.
 
 \textbf{ESC-GNN is more powerful than 2-WL.} Using the MPNN~\citep{xu2018powerful} as the backbone network, it can be as powerful as 2-WL in terms of distinguishing non-isomorphic graphs. However, there exist pairs of graphs, e.g., the 4*4 Rook Graph and the Shrikhande Graph that can be distinguished by ESC-GNN but not 2-WL. Therefore, ESC-GNN is strictly more powerful than the 2-WL.

 \end{proof}
 
\section{The Proof of Theorem 4.2}
Below we show the counting power of ESC-GNN in terms of subgraph counting. 

\begin{figure*}[btp]
\begin{center}
\centerline{\includegraphics[width=0.8\columnwidth]{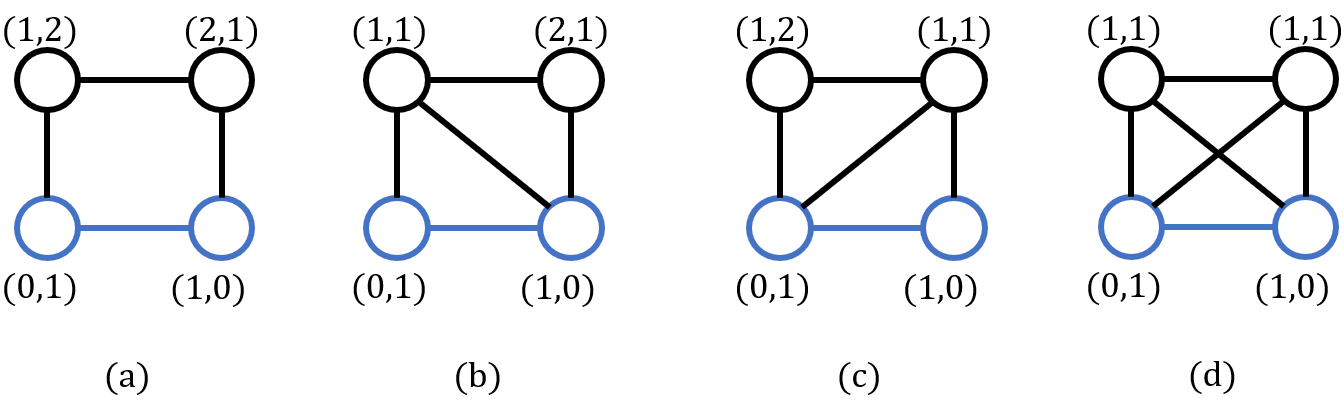}}
\caption{Examples of 4-cycles that pass the rooted edges. In these figures, the rooted 2-tuples are colored blue.}
\label{fig:4cycle}
\end{center}
\end{figure*}



\begin{theorem}
\label{th:count}
In terms of subgraph counting, ESC-GNN can count (1) up to 4-cycles; (2) up to 4-cliques; (3) stars with arbitrary sizes; (4) up to 3-paths.
\end{theorem}

\begin{figure*}[btp]
\begin{center}
\centerline{\includegraphics[width=0.5\columnwidth]{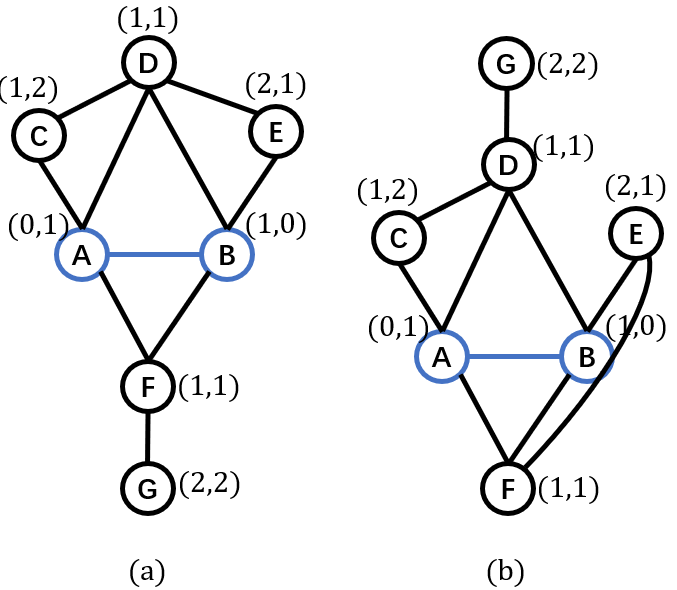}}
\caption{Examples where ESC-GNN cannot subgraph-count 5-cycles. In these figures, the rooted 2-tuples are colored blue.}
\label{fig:5cycle_sub}
\end{center}
\end{figure*}

\begin{proof}
\textbf{Clique counting.} The number of 3-cliques is the number of nodes with the shortest path distance ``1" to both rooted nodes; the number of 4-cliques is the number of edges labeled $(1,1,1,1)$ (definition see the edge-level distance encoding in Section 4 in the main paper, and example see Figure~\ref{fig:4cycle}(d)). Therefore, ESC-GNN can count these types of cliques. In terms of 5-cliques, 4-WL cannot count them according to Theorem~\ref{th:disab}, therefore subgraph GNNs rooted on edges with MPNN as the backbone GNN cannot count 5-cliques according to Proposition 3.5 in the main paper. Then according to Proposition 4.1 in the main paper, ESC-GNN also cannot count 5-cliques.

\textbf{Cycle counting.} the counting of 3-cycles is the same as 3-cliques. In terms of counting 4-cycles, there are basically 4 different situations where 4-cycles exist, examples are shown in Figure~\ref{fig:4cycle}. Note that figures (a),(b), and (c) contain one 4-cycle, respectively, while figure (d) contains two 4-cycles that pass the rooted edge. Therefore the number of 4-cycles is the weighted sum of the number of edges with labels $(1,2,2,1)$, $(1,1,2,1)$, $(1,2,1,1)$, $(1,1,1,1)$.

In terms of 5-cycle subgraph-counting, we provide a counter-example in Figure~\ref{fig:5cycle_sub}. In Figure~\ref{fig:5cycle_sub}(a), there is one 5-cycle $ABEDC$ that passes $AB$, while there is no 5-cycle that passes $AB$ in Figure~\ref{fig:5cycle_sub}(b). Considering that the degree information and the distance information is the same for the pair of graphs, ESC-GNN cannot differentiate them. Therefore, ESC-GNN cannot subgraph-count 5-cycles.

\textbf{Path Counting.} Note that the definition of paths here is different from other substructures. It denotes the number of edges within the path instead of the number of nodes within the path. Here, we slightly extend the use of 2-tuples in ESC-GNN by considering not only 2-tuples with edges but also 2-tuples without edges.

In terms of counting 2-paths or 3-paths between 2 nodes, it is equal to counting 3-cycles or 4-cycles, between edges, respectively. We have proven that ESC-GNN can count these cycles, therefore ESC-GNN can count such edges.

\textbf{Star counting.} We can decompose the graph-level star counting problem to 2-tuples by considering the first node of each 2-tuple as the root of stars. Examples are shown in Figure~\ref{fig:star}. We advocate that the number of stars is easily encoded by the number of nodes whose shortest path distance is 1 to the first rooted node. Denote the number of nodes with the shortest path distance 1 to the first rooted node as $N'$ (including the second node), then the number of $p$-stars is $C_{N'-1}^{p-2}$. A similar proof is provided by~\citep{chen2020can}.

\end{proof}






\begin{figure*}[btp]
\begin{center}
\centerline{\includegraphics[width=0.7\columnwidth]{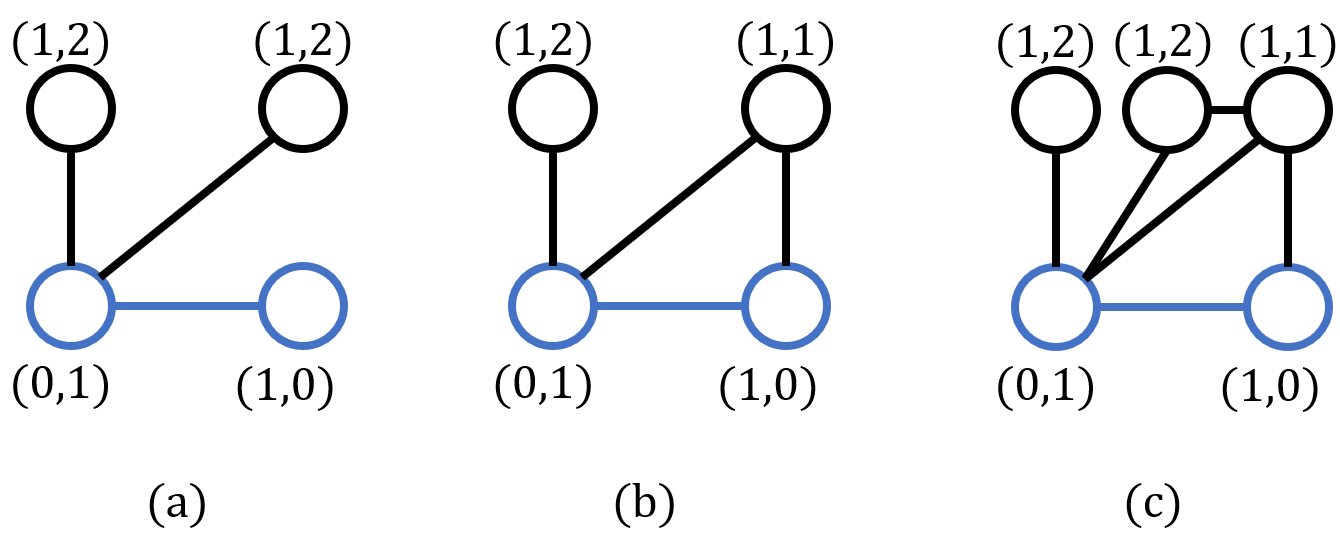}}
\caption{Examples of stars that pass the rooted edges. In these figures, the rooted 2-tuples are colored blue.}
\label{fig:star}
\end{center}
\end{figure*}


\section{The Proof of Theorem 4.3}

Below we show the counting power of ESC-GNN in terms of induced subgraph counting. We restate the Theorem 4.3 as follows.

\begin{theorem}
\label{th:count_ind}
In terms of induced subgraph counting, ESC-GNN can count (1) up to 4-cycles; (2) up to 4-cliques; (3) up to 4-stars; (4) up to 3-paths.
\end{theorem}


\begin{figure*}[btp]
\begin{center}
\centerline{\includegraphics[width=0.6\columnwidth]{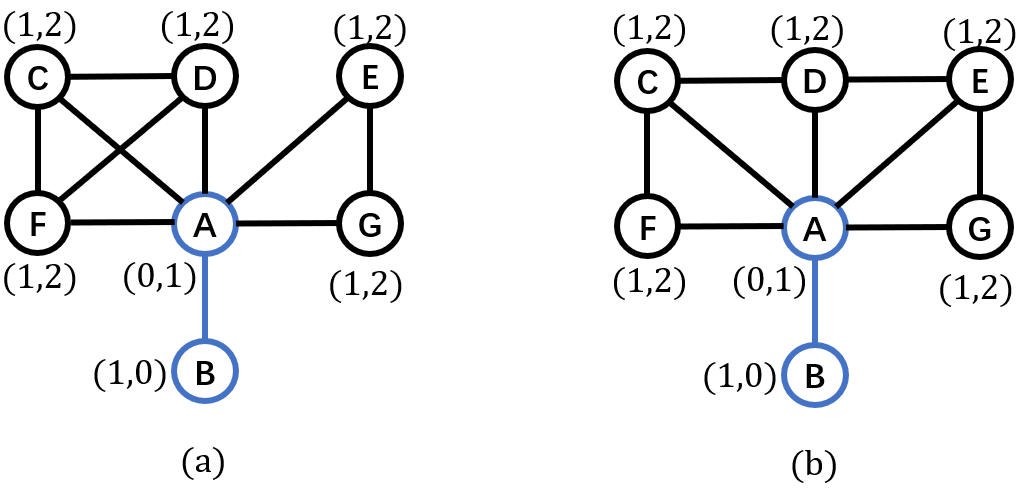}}
\caption{Examples where ESC-GNN cannot induced-subgraph-count 5-stars. In these figures, the rooted 2-tuples are colored blue.}
\label{fig:5star}
\end{center}
\end{figure*}

\begin{proof}
\textbf{Clieque counting.} Since cliques are fully connected substructures, the proof of cliques is the same as the proof of subgraph counting.

\textbf{Cycle counting.} For cycles, the number of 3-cycles is the same as 3-cliques. As for 4-cycles, we only need to consider the situation shown in Figure~\ref{fig:4cycle}(a), where the number of $(1,2,2,1)$ edges reflects the number of 4-cycles. For induced-subgraph-counting 5-cycles, in Figure~\ref{fig:5cycle}(a), there is no 5-cycle that pass $AB$, while in Figure~\ref{fig:5cycle}(b), there is one 5-cycle $ABEDC$ that passes $AB$. However, ESC-GNN cannot differentiate the two graphs since the degree information and the distance information is the same. Therefore, ESC-GNN cannot induced-subgraph-count 5-cycles. It can serve as the same example for not counting 4-paths.

\textbf{Path counting.} The proof of paths is actually the same as Theorem~\ref{th:count}.

\textbf{Star counting.} In terms of 3-stars and 4-stars, we only need to consider the situation shown in Figure~\ref{fig:star}(a), where the number of $(1,2)$ nodes encodes the number of stars, i.e., denote the number of $(1,2)$ nodes as $N'$, the number of $p$-stars ($p \leq 4$) is $C_{N'}^{p-2}$. 

For 5-stars, a pair of examples are shown in Figure~\ref{fig:5star}. These two graphs will be assigned the same structural embedding since the node degree information, and the distance information among these two graphs are the same. Therefore, ESC-GNN cannot differentiate the two graphs. However, Figure~\ref{fig:5star}(a) contains no 5-star that passes the rooted 2-tuple $AB$, while Figure~\ref{fig:5star}(b) contains one 5-star ($BAFDG$) that passes the rooted 2-tuple $AB$.
\end{proof}

\begin{figure*}[btp]
\begin{center}
\centerline{\includegraphics[width=0.55\columnwidth]{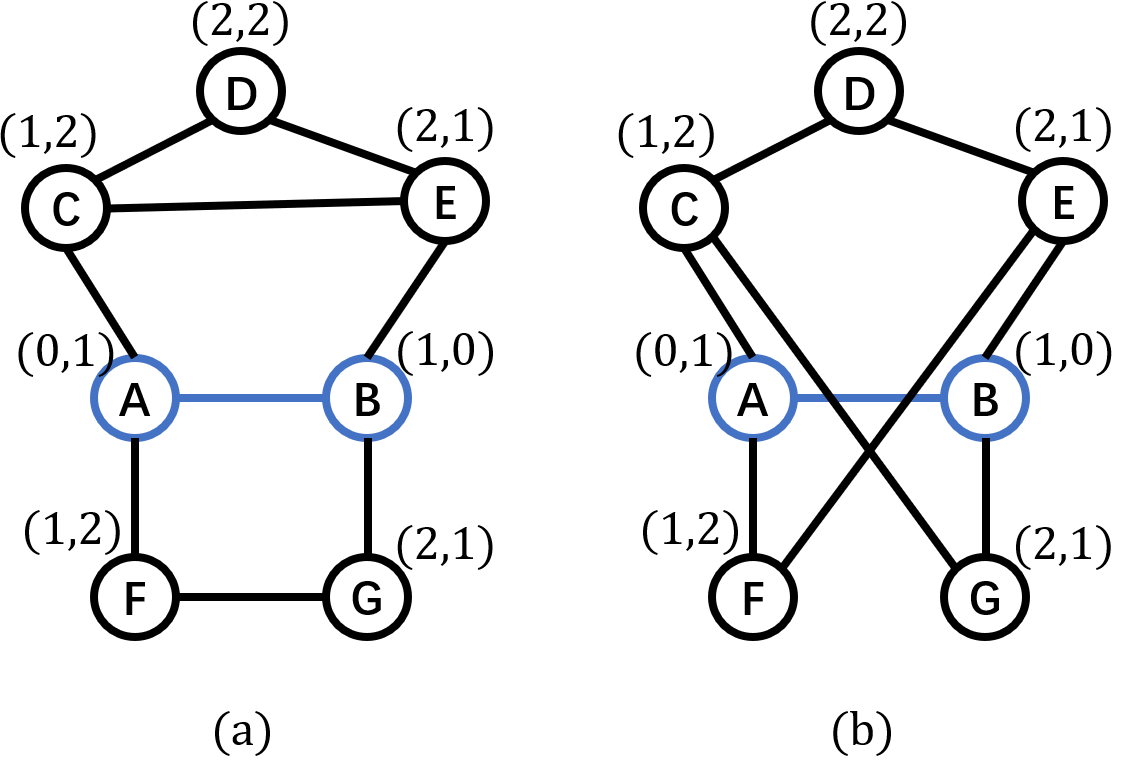}}
\caption{Examples where ESC-GNN cannot induced-subgraph-count 5-cycles. In these figures, the rooted 2-tuples are colored blue.}
\label{fig:5cycle}
\end{center}
\end{figure*}

\section{The Proof of Theorem 4.5}

Here we restate Theorem 4.5 as follows: 


\begin{theorem}
\label{th:reg}
Consider all pairs of $r$-regular graphs with $n$ nodes, let $3\leq r < (2log2n)^{1/2}$ and $\epsilon$ be a fixed constant. With the hop parameter $h$ set to $\lfloor(1/2 + \epsilon)\frac{log2n}{log(r-1)}\rfloor$, there exists an ESC-GNN that can distinguish $1-o(n^{-1/2})$ such pairs of graphs.
\end{theorem}

Recall that we encode the distance information as augmented structural features for ESC-GNN. For the input graph $G$, denote the node-level distance encoding on $h$-hop subgraphs on 2-tuple $(u,v)$ as $D_{(u,v),G}^{h,\text{node}}$. Note that we can naturally transfer the encoding on 2-tuples to node by: ($D_{v,G}^{h,\text{node}} = D_{(v,v),G}^{h,\text{node}}$). We then introduce the following lemma.

\begin{lemma}
\label{lemma:reg}
For two graphs $G_1 = (V_1, E_1)$ and $G_2 = (V^2, E^2)$ that are randomly and independently sampled from $r$-regular graphs with n nodes ($3 \leq r < (2log2n)^{1/2}$). Select two nodes $v_1$ and $v_2$ from $G_1$ and $G_2$ respectively. Let $h = \lfloor(1/2 + \epsilon)\frac{log2n}{log(r-1)}\rfloor$ where $\epsilon$ is a fixed constant, $D_{v_1,G_1}^{h,\text{node}} = D_{v_2, G_2}^{h, \text{node}}$ with probability at most $o(n^{-3/2})$.
\end{lemma}

\begin{proof}
The proof follows from~\citep{bollobas1982distinguishing, feng2022powerful}. As $D_{v,G}^{h,\text{node}}$ stores exactly the same information as the node configuration used in~\citep{feng2022powerful}, therefore the proof is exactly the same as the proof of Lemma 1 in~\citep{feng2022powerful}. 

Based on Lemma~\ref{lemma:reg}, we can prove Theorem~\ref{th:reg}. Given node $v_1 \in V_1$, compare $D_{v_1,G_1}^{h,\text{node}}$ with each $D_{v_2,G_2}^{h,\text{node}}$ where $v_2 \in V_2$. The probability that $D_{v_1,G_1}^{h,\text{node}} \neq D_{v_2,G_2}^{h,\text{node}}$ for all possible $v_2 \in V_2$ is $1-o(n^{-3/2})*n = 1 - o(n^{-1/2})$. Therefore, ESC-GNN can distinguish $1-o(n^{-1/2})$ such pairs of graphs.
\end{proof}

\section{Additional Discussion on the Expressive Power of ESC-GNN}

In previous sections, we mainly discuss the expressive power enhanced by the proposed structural encoding. In this section, we shift to the integration of the global message passing layer and elucidate its substantial contribution to the expressive power. This is exemplified below through a toy example: consider graph $G$, composed of a 4-cycle and a 1-path (edge), and graph $H$, comprising a 5-path (a path with 6 nodes). When extracting 1-hop subgraphs for each node, both $G$ and $H$ yield two 1-paths and four "V graphs" (3 nodes and 2 edges). Consequently, the subgraph representation fails to distinguish between these two graphs. However, the introduction of a global message passing layer enables differentiation due to the distinct connections among these subgraphs. This distinction is applicable to both subgraph GNNs and our framework. A more comprehensive analysis can be found in~\citep{rattan2023weisfeiler}, which substantiates that the global message passing can enhance expressiveness regardless of the choice of the hop parameter.

Additionally, We conduct an ablation study on the substructure counting dataset by removing the global message passing layer from ESC-GNN. The resulting new model, denoted as "(-MP)", relies solely on aggregating the structural encoding to derive the final prediction. Given the absence of additional node attributes, the new model can be easily implemented as ESC-GNN with a single message passing layer that aggregates the edge-level structural embedding to obtain the node-level prediction. The result is presented in Table~\ref{tab:count_ab_mp}. As evident from the table, ESC-GNN consistently outperforms the new model across all benchmarks, demonstrating that the global message passing layer significantly boosts the representation power of ESC-GNN.

\begin{table*}
	\centering
	\caption{Ablation study on the proposed structural embedding (norm MAE).} 
    \label{tab:count_ab_mp}
	\scalebox{0.72}{
	\begin{tabular}{lccccccccc}
	\hline\noalign{\smallskip}
	Dataset & Tailed Triangle & Chordal Cycle & 4-Clique & 4-Path & Triangle-Rectangle & 3-cycles & 4-cycles & 5-cycles & 6-cycles \\
	\noalign{\smallskip}\hline\noalign{\smallskip}
	MPNN & 0.3631 & 0.3114 & 0.1645 & 0.1592 & 0.2979 &  0.3515 & 0.2742 & 0.2088 & 0.1555\\
 	\noalign{\smallskip}\hline\noalign{\smallskip}
	ESC-GNN & \cellcolor{yellow}0.0052 & 0.0169 & \cellcolor{yellow}0.0064 & 0.0254 & 0.0178 & \cellcolor{yellow}0.0074 &\cellcolor{yellow}0.0044 & 0.0356& 0.0337\\
    \noalign{\smallskip}\hline\noalign{\smallskip}
     (- MP) & 0.062 & 1.9355 & 0.4124& 0.0271 &0.0563 & 0.1508 & \cellcolor{yellow}0.0050 & 0.0996 &  0.0443 \\
	\hline\noalign{\smallskip}
    \end{tabular}}
\end{table*}

\section{Experimental Details}

\textbf{Stastics of Datasets.} The statistics of all used datasets are available in Table~\ref{tab:stat}.

\begin{table*}
	\centering
	\caption{Statistics of the used datasets.}
	\label{tab:stat} 
	\scalebox{1.0}{
		\begin{tabular}{lccccc}
			\hline\noalign{\smallskip}
			Dataset & Graphs & Avg Nodes & Avg Edges & Task Type & Metric\\
			\noalign{\smallskip}\hline\noalign{\smallskip}
			MUTAG & 188 & 17.9 & 19.8 & Graph Classification & ACC\\
   			PTC-MR &  349 & 14.1 & 14.5 & Graph Classification & ACC\\
			ENZYMES &  600 & 32.6 & 62.1 & Graph Classification & ACC \\
			PROTEINS & 1113 & 39.1 & 72.8 & Graph Classification & ACC \\
			IMDB-BINARY & 1000 & 19.8 & 96.5 & Graph Classification & ACC\\
			ZINC-12k & 12000 & 23.2 & 24.9 & Graph Regression & MAE\\
			ogbg-molhiv & 41127 & 25.5 & 27.5 & Graph Classification & AUC-ROC\\
            ogbg-molpcba & 437929 & 26.0 & 28.1 & Graph Classification & AP\\
            QM9 & 129433 & 18.0 & 18.6 & Graph Regression & MAE\\
            Synthetic & 5000 & 18.8 & 31.3 & Node Regression & MAE\\
			\noalign{\smallskip}
			\hline
			\noalign{\smallskip}
	\end{tabular}}
\end{table*}

\textbf{Experimental details.} The baselines and data splittings of our experiments follow from existing works~\citep{zhang2021nested, huang2023boosting} and the standard data split setting. For ESC-GNN, we adopt GIN~\citep{xu2018powerful} as the backbone GNN. In the structural embedding, we use both the shortest path distance and the resistance distance~\citep{lu2011link} as the distance feature. For the hop parameter $h$ we search between 1 to 4, and report the best results. Following existing works~\citep{huang2023boosting}, we use Adam optimizer as the optimizer, and use plateau scheduler with patience 10 and decay factor 0.9. On most datasets, the learning rate is set to 0.001, and the hidden embedding dimension is set to 300. The training epoch is set to 2000 for counting substructures, 400 for QM9, 1000 for ZINC, and 150 for the OGB dataset. Most of the experiments are implemented with two Intel Core i9-7960X processors and 2 NVIDIA 3090 graphics cards. Others (e.g., experiments on the TU dataset) are implemented with two Intel Xeon Gold 5218 processors and 10 NVIDIA 2080TI graphics cards. 

\section{Evaluation on Real-World Datasets}
\label{subsec:real}

\textbf{Molecule Dataset.} We evaluate ESC-GNN on various popular real-world molecule datasets, including ZINC~\citep{dwivedi2020benchmarking} and the OGB dataset~\citep{hu2020open}. ZINC is a dataset of chemical compounds, and the task is graph regression. For the OGB dataset, we use ogbg-molhiv and ogbg-molpcba for evaluation. ogbg-molhiv contains 41K molecules with 2 classes, and ogbg-molpcba contains 438 molecules with 128 classes. The task is to predict to which these molecules belong. We follow the standard evaluation metric and the dataset split, and report the result in Table~\ref{tab:mol}.


\begin{table}
\centering
\caption{Evalation on ZINC and OGB datasets.} 
\label{tab:mol}
\scalebox{0.8}{
\begin{tabular}{lccc}
\hline\noalign{\smallskip}
Dataset & OGBG-HIV (AUCROC) & OGBG-PCBA (AP) & ZINC \\
\noalign{\smallskip}\hline\noalign{\smallskip}
GIN~\citep{xu2018powerful} & 77.07$\pm$1.49 & 27.03$\pm$0.23 & 0.163 \\
PNA~\citep{corso2020principal} & 79.05$\pm$1.32 & 28.38$\pm$0.35 & 0.188\\
GSN~\citep{bouritsas2022improving} & 77.99$\pm$0.01 & - & 0.115\\
DGN~\citep{beaini2021directional} & 79.70$\pm$0.97 & 28.85$\pm$0.30 & 0.168 \\
CIN~\citep{bodnar2021weisfeiler1} & \textbf{80.94$\pm$0.57} & - & \textbf{0.079} \\
NGNN~\citep{zhang2021nested} & 78.34$\pm$1.86 & 28.32$\pm$0.41 & 0.111\\
GIN-AK+~\citep{zhao2022stars} & 79.61$\pm$1.19 & \textbf{29.30$\pm$0.44} & 0.080\\
OSAN~\citep{qianordered2022} & - & - & 0.126\\ 
SUN~\citep{frascaunderstanding2022} & 80.03$\pm$0.55 &  - & 0.083 \\
DSS-GNN~\citep{bevilacqua2022equivariant} & 76.78$\pm$1.66 & - & 0.102 \\
I$^2$-GNN~\citep{huang2023boosting} & 78.68$\pm$0.93 & - & 0.083 \\
Graph Transformer~\citep{rampavsek2022recipe} & 77.40$\pm$1.77 & 27.51$\pm$0.28 & 0.113 \\
\noalign{\smallskip}\hline\noalign{\smallskip}
ESC-GNN & 78.62$\pm$1.06 & 28.16$\pm$0.31 & 0.075$\pm$0.002\\
\noalign{\smallskip}\hline\noalign{\smallskip}
\end{tabular}}
\end{table}

\textbf{TU datasets.} We evaluate the performance of ESC-GNN on the TU datasets~\citep{morris2020tudataset}. The experimental settings follow~\citep{zhang2021nested} for more consistent evaluation standards. Specifically, we uniformly use the 10-fold cross validation framework, with the split ratio of training/validation/test set 0.8/0.1/0.1. The results are available in Table~\ref{tab:tu}. In the table, ESC-GNN (h2) and ESC-GNN (h3) denote ESC-GNN with the hop parameters setting to 2 and 3. 

\subsection{Results} 

\textbf{Comparison with backbone models.} On OGBG-MolHIV, OGBG-MolPCBA, and the TU dataset, we use GIN~\citep{xu2018powerful} as the backbone model. While on ZINC, we use a plain graph transformer without positional encodings from~\citep{rampavsek2022recipe} as the backbone model. The experimental results compared with these backbones demonstrate that the proposed structural information not only enhances the theoretical representation power but also improves the empirical representation power.

\textbf{Comparison with subgraph GNNs.} When compared to subgraph GNNs rooted at 2-tuples, such as I$^2$-GNN, ESC-GNN performs slightly worse or comparably on these benchmark tasks. This empirical observation provides evidence that the proposed structural embedding effectively captures valuable information from subgraph GNNs, benefiting downstream tasks. It is worth noting that I$^2$-GNN generally outperforms node-based subgraph GNNs like NGNN and GIN-AK+, suggesting that the theoretical advantages gained from subgraph selection and aggregation policies contribute to the empirical representation power. Furthermore, among the subgraph GNNs, those that introduce interactions between subgraphs, such as SUN, GIN-AK+, and OSAN, tend to perform better compared to subgraph GNNs that do not incorporate such interactions, like NGNN and ID-GNN. This observation indicates that subgraph interaction also plays a role in enhancing the representation power of subgraph GNNs. Exploring additional techniques to incorporate this interaction can be considered for future research.

\textbf{Comparison with GNNs that incorporate the substructure information.} GSN~\citep{bouritsas2022improving} incorporates the number of specific substructures as augmented features. In contrast to GSN, our framework does not explicitly encode the number of substructures. Instead, we encode the more general distance information of subgraphs, which enables plain GNNs to count many substructures without the need to manually determine which substructures to include. Regarding the real-world dataset, GSN achieves a MAE score of 0.115 on ZINC, whereas ESC-GNN attains a MAE score of 0.096. Similarly, utilizing the same backbone (GIN), GSN records an AUCROC score of 77.99 on OGBG-HIV, whereas ESC-GNN demonstrates a higher AUCROC score of 78.62. These outcomes underscore the superiority of our proposed model.

\begin{table}
\centering
\caption{Experiments on TU, Accuracy as the evaluation metric.} 
\label{tab:tu}
\begin{tabular}{lcccccccc}
\hline\noalign{\smallskip}
Dataset & MUTAG & PTC-MR & PROTEINS & ENZYMES & IMDB-B \\
\noalign{\smallskip}\hline\noalign{\smallskip}
GIN & 84.5$\pm$8.9 & 51.2$\pm$9.2 & 70.6$\pm$4.3 & 38.3$\pm$6.4 & 73.3$\pm$4.7 \\
PPGN & 84.7$\pm$8.2 & 55.0$\pm$6.4 & 74.8$\pm$3.3 & 55.0$\pm$6.4 & 71.5$\pm$5.4\\
NGNN & 87.9$\pm$8.2 & 54.1$\pm$7.7 & 73.9$\pm$5.1 & 29.0$\pm$8.0 & 73.1$\pm$5.7 \\
GIN-AK+ & \textbf{88.8$\pm$4.0} & 60.5$\pm$8.0 & 75.5$\pm$4.4 & \textbf{58.9$\pm$6.2} & 72.4$\pm$3.7\\ 
SUN & 86.1$\pm$6.0 & 60.2$\pm$7.2 & 72.1$\pm$3.8 & 16.7$\pm$0.0 & \textbf{73.7$\pm$2.9}\\
I$^2$-GNN & 87.9$\pm$4.3 & \textbf{61.4$\pm$8.7} & 74.8$\pm$2.9 & 40.3$\pm$6.7 & 73.6$\pm$4.0\\
\noalign{\smallskip}\hline\noalign{\smallskip}
ESC-GNN (h2) & 86.2$\pm$7.9 & 52.9$\pm$6.4 & 73.3$\pm$4.1 & 53.2$\pm$8.1 & 72.0$\pm$6.0\\
ESC-GNN (h3) & 85.6$\pm$7.9 & 56.4$\pm$6.9 & \textbf{76.0$\pm$4.5} & 43.3$\pm$6.0 & \textbf{73.7$\pm$4.8}\\
\noalign{\smallskip}\hline\noalign{\smallskip}
\end{tabular}
\end{table}

\section{Ablation Study}
\label{subsec:ablation}
We evaluate the effectiveness of each part of the proposed structural embedding on the substructure counting dataset. For the proposed three types of structural embedding (the degree encoding, the node-level distance encoding, and the edge-level distance encoding), we delete one of them from the original structural embedding every time and report the results in Table~\ref{tab:count_ab}. We observe that after removing the three types of embedding, ESC-GNN performs worse compared with its original version, especially after removing the edge-level distance encoding. This is consistent with our theoretical results: in Theorem~\ref{th:sub}, we show that the proposed structural embedding contains key information for the counting power of subgraph GNNs; in Theorem~\ref{th:count} and Theorem~\ref{th:count_ind}, we show that the edge-level distance information directly encodes the number of certain types of substructures.

\begin{table*}
	\centering
	\caption{Ablation study on the proposed structural embedding (norm MAE).} 
    \label{tab:count_ab}
	\scalebox{0.72}{
	\begin{tabular}{lccccccccc}
	\hline\noalign{\smallskip}
	Dataset & Tailed Triangle & Chordal Cycle & 4-Clique & 4-Path & Triangle-Rectangle & 3-cycles & 4-cycles & 5-cycles & 6-cycles \\
	\noalign{\smallskip}\hline\noalign{\smallskip}
	MPNN & 0.3631 & 0.3114 & 0.1645 & 0.1592 & 0.2979 &  0.3515 & 0.2742 & 0.2088 & 0.1555\\
 	\noalign{\smallskip}\hline\noalign{\smallskip}
	ESC-GNN & \cellcolor{yellow}0.0052 & 0.0169 & \cellcolor{yellow}0.0064 & 0.0254 & 0.0178 & \cellcolor{yellow}0.0074 &\cellcolor{yellow}0.0044 & 0.0356& 0.0337\\
    \noalign{\smallskip}\hline\noalign{\smallskip}
     (- degree) &0.0121 & 0.0492& 0.0106& 0.0322 & 0.0349 &0.0342 &0.0144 &0.0513 & 0.0652\\
    (- node-level dist) & 0.0382 & 0.0344 & 0.0222 &0.0428 & 0.0256 & 0.0157 & 0.0261 & 0.0492 &0.0608  \\
    (- edge-level dist) & 0.0208&0.2811 &0.0497 & 0.0584&0.185 & 0.2617&0.2244 & 0.1654& 0.1364\\
	\hline\noalign{\smallskip}
    \end{tabular}}
\end{table*}

\section{Evaluation on the Space Cost}

With regards to the preprocessing time, our approach's preprocessing time is comparable to that of I$^2$-GNN as shown in Table 3 in the main paper since we use their code to extract subgraph information. Although we require a small amount of extra time to extract and preprocess the structural embeddings, we believe that this is a reasonable trade-off since the actual model running time is reduced by orders. When compared to the total running time of subgraph GNNs, the preprocessing time is negligible. Therefore, our ESC-GNN's total running time is less than 1\% of that of I$^2$-GNN on ogbg-hiv.


As for the storage of structural embeddings, we only need to store a vector of integer indices for each structural embedding, as illustrated in Figure 1 in the main paper, without really storing any dense high-dimensional vectors. Therefore, the storage is still manageable. Specifically, when feeding the indices to the backbone model, we use a learnable matrix (which is a model parameter) to transform the integer index vector into dense embeddings. For example, to extract the degree information of the first subgraph in Figure 1 in the main paper, we use a learnable matrix $W \in \mathbb{R}^{4\times h}$ and compute the degree information with $W * [0;0;4;0]$, where $*$ denotes sparse matrix multiplication (which is very fast for sparse matrices), and the sparse vector $[0;0;4;0]$ is what we need to store. This approach requires relatively small storage space.

We have also included the space cost of our approach in Table~\ref{tab:space}. In these datasets, ESC-GNN requires much less storage space than I$^2$-GNN while slightly more space than NGNN. These results demonstrate that our approach for storing structural embeddings offers excellent storage performance.

\begin{table*}
	\centering
	\caption{Evaluation on the space cost.}
	\label{tab:space} 
	\scalebox{1.0}{
		\begin{tabular}{lccccc}
			\hline\noalign{\smallskip}
			Dataset & OGBG-HIV &ZINC & QM9\\
			\noalign{\smallskip}\hline\noalign{\smallskip}
			Original & 159MB & 16.2MB & 281MB\\
   			NGNN & 2.32GB & 218MB & 2.90GB\\
			I$^2$-GNN & 5.95GB & 627MB & 8.25GB \\
			ESC-GNN & 2.57GB & 427MB & 4.46GB\\

			\noalign{\smallskip}
			\hline
			\noalign{\smallskip}
	\end{tabular}}
\end{table*}

\section{Counting Substructures on Other Datasets}

In order to evaluate the counting power of the proposed model, we conducted experiments on the ZINC dataset by generating the cycle counting task and reporting the MAE result in Table~\ref{tab:cnt_zinc}. To count cycles, we used the simple-cycle function from \url{https://networkx.org/documentation/stable/reference/algorithms/generated/networkx.algorithms.cycles.simple_cycles.html}. Surprisingly, we found that all methods achieved much better MAE scores than those reported in Table 1 of the main paper. This may be due to the fact that many graphs in ZINC contain few cycles, resulting in a small MAE value.

In general, the reported results are consistent with Table 1 of the main paper. ESC-GNN performs better in counting 3-cycles and 4-cycles, and performs worse on 5-cycles and 6-cycles. This observation is consistent with Theorem~\ref{th:count}. Furthermore, it outperforms MPNN, performs comparably to GIN-AK+, and is outperformed by I$^2$-GNN, which is consistent with our theoretical results, showing that the proposed structural embedding can extract valuable information from the subgraph GNNs.

\begin{table*}
	\centering
	\caption{Evaluation on Counting Substructures on ZINC (norm MAE).}
	\label{tab:cnt_zinc} 
	\scalebox{1.0}{
		\begin{tabular}{lccccc}
			\hline\noalign{\smallskip}
			Tasks & 3-cycle & 4-cycle & 5-cycle & 6-cycle \\

			\noalign{\smallskip}\hline\noalign{\smallskip}
			GNN & 0.0016 & 0.0030 & 0.0394 & 0.1442\\
   			GIN-AK+ & 0.0009 & 0.0064 & 0.0036 & 0.0057\\
			NGNN & 	0.0002 &	0.0001	& 0.0007 & 0.0012 \\
            I$^2$-GNN &0.0003 & 0.0001 & 0.0003 & 0.0007 \\

			ESC-GNN & 0.0008	&0.0004	&0.0058 &	0.0047\\

			\noalign{\smallskip}
			\hline
			\noalign{\smallskip}
	\end{tabular}}
\end{table*}

\section{Limitations and the Assets We Used}

\textbf{Limitations of the paper.} First, we have shown that the representation power of our model is bounded by 4-WLs and subgraph GNNs rooted on 2-tuples in terms of distinguishing non-isomorphic graphs and counting substructures.

Second, the proposed model may not reach a satisfying performance on benchmarks where the encoded substructures are of no use. Also, the proposed model may not suit high-order graphs where the neighbors of the nodes and edges are defined differently from the simple graphs. 

\textbf{The assets we used.} Our model is experimented on benchmarks from~\citep{dwivedi2020benchmarking, hu2020open, zhao2022stars, morris2020tudataset, abboud2021surprising, balcilar2021breaking, murphy2019relational, ramakrishnan2014quantum, wu2018moleculenet} under the MIT license.





\nocite{langley00}

\bibliography{example_paper}
\bibliographystyle{iclr2024_conference}